\documentclass{article}
\usepackage[margin=1in]{geometry}

\usepackage[round]{natbib}

\usepackage{amsmath}
\usepackage{amssymb}
\usepackage{amsthm}
\usepackage[colorlinks=true,linkcolor=blue,citecolor=blue]{hyperref}
\newtheorem{theorem}{Theorem}
\newtheorem*{theorem*}{Theorem}
\newtheorem{lemma}{Lemma}
\newtheorem{proposition}{Proposition}
\newtheorem{example}{Example}
\newtheorem*{lemma*}{Lemma}
\newtheorem{definition}{Definition}

\usepackage{style}

\makeatletter
\renewenvironment{proof}[1][\proofname]{\par
  \pushQED{\qed}%
  \normalfont \topsep6\p@\@plus6\p@\relax
  \trivlist
  \item[\hskip\labelsep
        \textbf{#1}]\ignorespaces
}{%
  \popQED\endtrivlist\@endpefalse
}
\makeatother

\title{Online Stackelberg Optimization via Nonlinear Control}
\author{William Brown\thanks{Columbia University \& Morgan Stanley MLR. Email: \texttt{w.brown@columbia.edu}.} \and Christos Papadimitriou\thanks{Columbia University. Email: \texttt{christos@cs.columbia.edu}} \and Tim Rougharden \thanks{Columbia University \& a16z crypto. Email: \texttt{tim.roughgarden@gmail.com}}}

\date{\today}

\begin{document}

\maketitle

\begin{abstract}
In repeated interaction problems with adaptive agents, our objective often requires anticipating and optimizing over the space of possible agent responses. We show that many problems of this form can be cast as instances of online (nonlinear) control which satisfy \textit{local controllability}, with convex losses over a bounded state space which encodes agent behavior, and we introduce a unified algorithmic framework for tractable regret minimization in such cases. When the instance dynamics are known but otherwise arbitrary, we obtain oracle-efficient $O(\sqrt{T})$ regret by reduction to online convex optimization, which can be made computationally efficient if dynamics are locally \textit{action-linear}. In the presence of adversarial disturbances to the state, we give tight bounds in terms of either the cumulative or per-round disturbance magnitude (for \textit{strongly} or \textit{weakly} locally controllable dynamics, respectively). Additionally, we give sublinear regret results for the cases of unknown locally action-linear dynamics as well as for  the bandit feedback setting. Finally, we demonstrate applications of our framework to well-studied problems including performative prediction, recommendations for adaptive agents, adaptive pricing of real-valued goods, and repeated gameplay against no-regret learners, directly yielding  extensions beyond prior results in each case. 
\end{abstract}

\section{Introduction}\label{sec:intro}

Machine learning problems involving strategic or adaptive agents are commonly framed as 
Stackelberg games, wherein the leader aims to commit to an optimal strategy in anticipation of the follower's best response. This approach has been effectively applied to challenges ranging from performative feature manipulation \citep{DBLP:journals/corr/HardtMPW15,SCrevealed,DBLP:journals/corr/abs-2002-06673,PPregret} and optimal pricing \citep{DBLP:journals/corr/RothUW15,DBLP:journals/corr/DaskalakisS15,pmlr-v108-nedelec20a} to resource allocation in security games \citep{NIPS2014_cc1aa436,10.1145/2764468.2764478,ALCANTARAJIMENEZ2020106695} and learning in tabular games \citep{Letchford2009LearningAA,Peng2019LearningOS,lauffer2022noregret,10.5555/3545946.3598695}, often with a regret minimization objective. Additionally, several of these settings have been independently extended to account for agents that may update their strategies gradually over time rather than optimally responding in each round \citep{DBLP:journals/corr/abs-2106-12529,brown2022performative,DBLP:journals/corr/abs-1711-09176,deng2019strategizing,brown2023learning}. Despite their conceptual similarities, these problems have largely been approached as distinct areas of study, each with their own growing body of techniques. Our aim in this work is to offer a unifying perspective and algorithmic approach for problems of this form, through the lens of online control.

For the broad family of online ``Stackelberg-style'' optimization problems, the language of control is quite natural to adopt: we are navigating a dynamical system where states corresponding to agent strategies evolve as a function of our own actions, and where objectives which consider best-response stability can be expressed in terms of the stationary behavior of this system. Our results consider a general class of online control instances for representing such problems, which we introduce in Section \ref{sec:model}, and in Section \ref{sec:algorithm} we give a sequence of no-regret algorithms for these instances satisfying a range of robustness properties. In Section \ref{sec:apps}, we show that several online optimization problems involving adaptive agents, including variants of online performative prediction (as in \cite{kumar2022online}), online recommendations  (as in \cite{agarwal2023online}), adaptive pricing (as in \cite{DBLP:journals/corr/RothUW15}), and learning in time-varying games (as in \cite{anagnostides2023convergence}) can be embedded in our framework and solved by our algorithms.

While there has been a great deal of recent progress in online linear control, yielding algorithms which can optimize over stabilizing linear policies even with general convex costs, adversarial disturbances, and unknown dynamics \citep{agarwal2019online,DBLP:journals/corr/abs-2001-09254,cassel2022efficient,minasyan2022online}, the required assumptions and regret benchmarks for these algorithms do not always type-check with the settings we are interested in. For the examples we consider, we will often wish to allow for nonlinear dynamics (e.g.\ encoding an agent's utility function) and explicitly bounded spaces (e.g.\ via projection into the simplex), and we will seek to compete with regret benchmarks which correspond to stable responses by the agent. Unfortunately, as we show in Proposition \ref{prop:linear-policies}, the latter goal is incompatible with linear policies even under linear dynamics and in the absence of any disturbances: the performance of {\it every} linear policy can be $\Omega(T)$ worse than the best policy in the class of affine ``state-targeting'' policies. 

In contrast, the orthogonal set of assumptions we identify enables tractable regret minimization even for {\it nonlinear} control problems and comports with the requirements of Stackelberg optimization across a wide range of settings, including the ability to compete with state-targeting policies. For convex and compact state and action spaces $\X$ and $\Y$, our first key assumption is that the dynamics $D(x,y) : \X \times \Y \rightarrow \Y$ satisfy a notion of {\it local controllability}. While local controllability is well-studied for continuous-time and asymptotic control \citep{10.2307/2296398,KUHN1989617,barberoliñán2013second,boscain2021local}, we are unaware of any prior applications to finite-time online optimization, and we adapt existing definitions to be appropriate for this setting. 
We say that $D(x,y)$ is \textit{strongly} locally controllable if every state in a fixed-radius ball around $y$ is reachable in a single round by an appropriate choice of $x$, and that $D(x,y)$ is \textit{weakly} locally controllable if the reachable radius around $y$ is allowed to vanish near the boundary of $\Y$. We also assume that our loss $f_t$ in each round is determined (or well-approximated by) an adversarially-chosen convex function depending only on the state $y_t$. 

When these conditions hold, we show in Theorem \ref{thm:oenftrl} that this is sufficient to obtain $O(\sqrt{T})$ regret with respect to the loss of the best fixed state, provided that dynamics are known and we have offline access to an oracle for non-convex optimization; the oracle call can be removed if dynamics are locally {\it  action-linear}, i.e.\ given by (or locally well-approximated by) a function linear in $x$ at each fixed $y$. If adversarial disturbances to the dynamics are present, our approach can be extended for both weakly (Theorem \ref{thm:oenftrlap-body}) and strongly (Theorem \ref{thm:oenftrluap-body}) locally controllable dynamics with additional regret scaling linearly in total disturbance magnitude, provided that each round's disturbance cannot be too large in the case of weak local controllability; we give lower bounds showing that each dependence on disturbance magnitude is tight.
The aforementioned results all extend to the case where the dynamics (absent disturbances) are given by a known but time-dependent function $D_{t}(x,y)$. If dynamics are unknown but time-invariant, and locally action-linear with appropriate regularity parameters, we obtain sublinear regret provided that a ``near-stabilizing'' action is known at $t=1$.
We additionally extend our approach to the bandit feedback setting, where we obtain $O(T^{3/4})$ regret.
In Section \ref{sec:apps} we show that each of the following, with appropriate assumptions, can be cast as a locally controllable instance with state-only convex surrogate losses:
\begin{itemize}
    \item {\bf Performative prediction:} Minimize prediction loss $\E_{z\sim p_t}f_t(x_t, z)$ for a classifier $x_t$, where the distribution $p_t$ in each round is updated according to the prior classifier and distribution.
    \item {\bf Adaptive recommendations:} Maximize the reward $f_t(i_t)$ when showing menus $K_t \subseteq [n]$ of size $k \ll n$ to an agent, whose choice $i_t \sim p(K_t, v_t)$ in each round depends on preferences which are influenced by choices in prior rounds (encoded in the ``memory vector'' $v_t$). 
    \item {\bf Adaptive pricing:} Maximize profit $\langle p_t, x_t \rangle$ - $c_t(x_t)$ for selling bundles of goods $x_t$ to an agent at prices $p_t$ and with costs $c_t$, where the agent's purchased bundle $x_t$ is a function of their utility function, consumption rate, and existing reserves.
    \item {\bf Repeated gameplay:} Maximize the reward $x_t^{\top}A_t y_t$ obtained from playing a sequence of time-varying games $(A_t, B_t)$ against a no-regret learning agent.
\end{itemize}
In each case, application of our algorithms from Section \ref{sec:algorithm} yields results which extend beyond the applicability regimes of prior work, such as by enabling relaxation of previous assumptions or a novel extension to adversarial or dynamic problem variants.

\subsection{Related Work}

\paragraph{Online control.} Much of the recent progress in online control \citep{agarwal2019online,NEURIPS2019_78719f11,cassel2022efficient,minasyan2022online} considers linear systems with general convex losses, benchmarking against a class of (``strongly stable'') fast-mixing linear policies introduced for linear-quadratic control \citep{onlineLQC} by leveraging the framework of ``OCO with memory'' \citep{anava2014online}.
Results have also been shown for nonlinear policy classes via neural networks \citep{chen2022provable}, and for nonlinear dynamics with oracles in episodic settings \citep{kakade2020information}, via approximation with random Fourier features \citep{9683670,luo2022sampleefficient}, via adaptive regret for time-varying linear systems \citep{gradu2022adaptive,minasyan2022online}, and via dynamic regret over actions in terms of disturbance ``attenuation'' \citep{muthirayan2022online}.  
For a further overview of online control and its historical context, see \cite{hazan2022introduction}.
In contrast to the bulk of prior work in which states and actions are bounded implicitly via policy stability notions, we consider state and action spaces which are bounded explicitly, as enabled by nonlinearity in dynamics (e.g.\ via projection, or range decay of dynamics near the boundary). These works also view disturbances as intrinsic to the system, and account for their influence directly in regret benchmarks (the ``optimal policy'' will face the same sequence of disturbances in hindsight, regardless of state). Within the context of Stackelberg optimization where a fixed protocol largely determines an agent's strategy updates, we view the role of disturbances as more akin to adversarial {\it corruptions} as considered in reinforcement learning \citep{pmlr-v134-lykouris21a,DBLP:journals/corr/abs-2106-06630}; while we incur linear dependence, our regret benchmarks are agnostic to alternate counterfactual disturbance sequences.

\paragraph{Strategizing against learners.} Initially formulated within the context of repeated auctions \citep{DBLP:journals/corr/abs-1711-09176}, a recent line of work has considered the problem of optimizing long-run rewards in a repeated game against a no-regret learner across a range of tabular and Bayesian settings \citep{deng2019strategizing,mansour2022strategizing,brown2023learning,zhang2023steering}. While bounds on attainable reward have been known in terms of the Price of Anarchy \citep{blum2008regret,DBLP:journals/corr/HartlineST15}, this sequence of results has highlighted important connections with Stackelberg equilibria: the Stackelberg value of the game is attainable on average against any no-regret learner, and it is the maximum attainable value against many common no-regret algorithms (such as no-swap learners, as shown by \cite{deng2019strategizing}). This theme has emerged in other simultaneous learning settings as well; 
notably, \cite{NEURIPS2021_812214fb} show that long-run outcomes in strategic classification are shaped by relative learning rates between parties, which can designate either as the Stackelberg leader.

\paragraph{Nested convex optimization.} The technique of identifying convex structure nested inside a more general problem  has been applied broadly across a range of online optimization settings \citep{NEURIPS2021_5631e6ee,Shen2023,flokas2019poincare}. For repeated interaction problems involving an agent with unknown utility, such as optimal pricing, \cite{DBLP:journals/corr/RothUW15} identify utility conditions under which the non-convex objective over prices becomes convex in the space of agent actions,
and where explorability properties resembling local controllability hold, which enables convex optimization by locally learning agent preferences; this ``revealed preferences'' approach has also been applied to strategic classification \citep{SCrevealed}. In recent work concerning recommendations for agents with history-dependent preferences \citep{AB22,agarwal2023online}, properties related to local controllability are leveraged to enable tractable optimization as well. We consider each of these settings as applications in Section \ref{sec:apps}.

\section{Model and Preliminaries}\label{sec:model}

Let $\X$ and $\Y$ be convex and compact subsets of Euclidean space, 
respectively denoting the action and state spaces, where we assume $\dim(\X) \geq \dim(\Y)$. Further, we assume that $\Y$ contains a ball of radius $r$ around the origin $\mathbf{0}$, and is contained in a ball of radius $R$ around the origin. 

An instance of our control problem consists of choosing a sequence of actions $\{x_t \in \X \}$ over $T$ rounds, which will yield a sequence of states $\{y_t \in \Y \}$, and we will incur losses determined by adversarially chosen functions $\{f_t\}$.
Let the initial state be $y_0 = \mathbf{0}$. 
In the basic version of our problem, upon choosing each $x_t$ for rounds $t \in [T]$, we observe the state update to
\begin{align*}
    y_{t} =&\; D(x_t, y_{t-1}),
\end{align*}
where $D : \X \times \Y \rightarrow \Y $ is an arbitrary continuous function which we refer to as the {\it dynamics} of our problem. 
We sometimes allow {\it disturbances} to the dynamics, where $y_{t} = D(x_t, y_{t-1}) + w_{t+1}$ for $\{w_t\}$ chosen adversarially. In some cases we allow {\it time-varying} dynamics $D : \X \times \Y \times [T] \rightarrow \Y $, where the dynamics in each round are denoted by $D_t(x_t, y_{t-1})$.

Here and in Section \ref{sec:algorithm}, we assume that our loss in round is given by $f_t(y_t)$, where each $f_t$ is a $L$-Lipschitz convex function revealed after playing $x_t$; we relax these assumptions for some of our applications in Section \ref{sec:apps}, e.g.\ to allow dependence on $x_t$ as well. We generally measure will performance with respect to the best fixed state, and the regret for an algorithm $\A$ yielding $\{y_t\}$ is
\begin{align*}
    \Reg_{T}(\A) = \sum_{t=1}^T f_t(y_t)  - \min_{y \in \Y} \sum_{t=1}^T {f}_t(y). 
\end{align*}
In Proposition \ref{prop:linear-policies}, we relate this benchmark to the class of ``state-targeting'' policies, which can sometimes be expressed by affine functions, and we compare their performance to linear policies.
Throughout, we use $\norm{\cdot}$ to donate the Euclidean norm, and we let $\B_{\epsilon}(y) = \{\hat{y} : \norm{y - \hat{y}} \leq \epsilon \}$  denote the norm ball of radius $\epsilon$ around $y$. We let $\Pi_{\Y}(\cdot)$ denote Euclidean projection into the set $\Y$; $\mathbf{u}_n$ denotes the uniform distribution over $n$ items, and $\Delta(n)$ denotes the probability simplex.

\subsection{Locally Controllable Dynamics}
A number of properties under the name ``local controllability'' have been considered for various continuous-time and asymptotic control settings \citep{10.2307/2296398,KUHN1989617,barberoliñán2013second,boscain2021local}, generally relating to the notion that all states in a neighborhood around a given state are reachable. We give two formulations of local controllability for our setting, which we take as properties of the dynamics $D$ holding over all inputs.

\begin{definition}[Weak Local Controllability]
For $\rho \in (0, 1]$, an instance $(\X, \Y, D)$ satisfies (weak) $\rho$-local controllability if for  any $y \in \Y$ and $y^* \in \B_{\rho \cdot \pi(y)}(y)$, there is some $x$ such that 
$D(x, y) = y^*$,
where $\pi(y) = \min_{\hat{y} \in \bd(\Y)} \norm{\hat{y} - y}$ is the distance from $y$ to the boundary of $\Y$.
\end{definition}

\begin{definition}[Strong Local Controllability] For $\rho > 0$, an instance $(\X, \Y, D)$ satisfies strong $\rho$-local controllability if for  any $y \in \Y$ and $y^* \in \B_{\rho}(y)\cap \Y$, there is some $x$ such that 
$D(x, y) = y^*$.
\end{definition}
We often refer to weak local controllability simply as local controllability.
This property ensures that there is always some action $x_t$ which results in the next state $y_{t}$ staying fixed at $y_{t-1}$, as well as some action which moves the state to any point in a surrounding ball; in the weak case, the size of the reachable ball is allowed to decay as $y_t$ approaches the boundary of $\Y$. The parameter $\rho$ controls the speed at which we can navigate the state space: when $\rho = 1$ in the weak case (or $\rho \geq R$ in the strong case), we can always immediately reach some point on the boundary of $\Y$, yet for $\rho$ close to zero we may only be able to move in a small neighborhood. 
Our results use local controllability to minimize regret over $\Y$ by reduction to online convex optimization. As we prove in Appendix \ref{sec:subopt-linear}, up to a quantifier alternation which vanishes as $\rho$ approaches $0$, a property of this form is essentially necessary: competing with the best state $y$ is impossible if we cannot remain in its neighborhood.
\begin{proposition}\label{prop:local-controllability}
    Suppose there is some $y \in \Y$ and values $\alpha, \beta > 0$ such that for all $\hat{y} \in \B_{\alpha}(y)$ and $x \in \X$, $D(x, \hat{y}) \notin \B_{\beta}(\hat{y})$. Then, there are losses such that $\Reg_T(\A) = \Omega(T)$ for any algorithm $\A$.
\end{proposition}

\subsection{States vs.\ Policies}\label{subsec:lin-policies}

While regret benchmarks in online control are typically expressed in terms of a reference class of policies, we note that there is a class of ``state-targeting'' policies which track the reward of fixed states (asymptotically, and up to the influence of disturbances), and which can be implemented if $D$ is known; we maintain the formulation in terms of fixed states for clarity with respect to our motivations for Stackelberg optimization. Existing no-regret algorithms for online control typically compete with linear policies, and choose actions each round by implementing policies which are linear in multiple past states (as in e.g.\ \cite{agarwal2019online}). Here, we show that all such policies can be arbitrarily suboptimal when compared to state-targeting policies, even for dynamics which are linear up to projection and with fixed convex losses over states, as they may yield actions and states which remain fixed at $\mathbf{0}$ in every round even if the optimal state is always immediately accessible under the dynamics. We prove Proposition \ref{prop:linear-policies} in Appendix \ref{sec:subopt-linear}.
\begin{proposition}\label{prop:linear-policies}
For an instance $(\X,\Y, D)$, let the class of state-targeting policies for $\hat{\Y} \subseteq \Y$ be given by $\mathcal{P}_{\hat{\Y}} = \{P_{\hat{y}} : \hat{y} \in \hat{\Y} \}$ where $P_{\hat{y}}(y) = \argmin_{\{x \in \X : D(x, y) \in \hat{\Y}\}} \norm{D(x, y) - \hat{y}}^2$. 
Define the regret of a policy class $\mathcal{P}$ as
\begin{align*}
    \Reg_T(\mathcal{P}) = \min_{P \in \mathcal{P}} \parens{ \sum_{t=1}^T f_t(y_t)  } - \min_{y \in \Y} \parens{ \sum_{t=1}^T {f}_t(y) }, 
\end{align*}
where $y_t$ is updated by playing $P$ at each round. 
For any $\rho$-locally controllable instance, there is a set $\hat{\Y} \subseteq \Y$ for which $\Reg_T(\mathcal{P}_{\hat{\Y}}) = {O}(\sqrt{T\rho^{-1}})$. Further,
for any class $\mathcal{P}_{\mathcal{K}}$ where each $K \in \mathcal{P}_{\mathcal{K}}$ is a matrix yielding actions $x_t = -Ky_{t-1}$, there is an instance where $\Reg_T(\mathcal{P}_{\mathcal{K}}) \geq \Omega(T)$ for $\rho = 1$.
\end{proposition}
If dynamics are linear up to projection with $D(x_t, {y_{t-1}}) = \Pi_{\Y}(By + Ax)$  for full-rank $A$, and $\dim(\X) = \dim(\Y)$, note that $P_{\hat{y}}(y) = A^{-1}(\hat{y} - By)$ implements any $P_{\hat{y}}$ for sufficiently large $\X$.

\section{No-Regret Algorithms for Locally Controllable Dynamics}
\label{sec:algorithm}
Here we give a sequence of no-regret algorithms satisfying a range of robustness properties. Our primary algorithm $\oenftrl$, presented in Section \ref{subsec:oenftrl}, operates over known time-varying dynamics without disturbances and requires an offline non-convex optimization oracle, and we identify conditions in Section \ref{subsec:ALD} which remove the oracle requirement. In Section \ref{subsec:disturbances} we give two algorithms, $\oenftrlap$ and $\oenftrluap$, which allow adversarial disturbances to weakly and strongly locally controllable dynamics, respectively. In Section \ref{subsec:unknown-dynamics} we extend $\oenftrl$ to accommodate unknown dynamics under appropriate regularity conditions (provided an initial ``approximately stabilizing'' action is known at $t=1$), and in Section \ref{subsec:bandit} we give an algorithm which obtains $O(T^{3/4})$ regret under bandit feedback.
 
\subsection{Nonlinear Control via Online Convex Optimization}\label{subsec:oenftrl}

When dynamics satisfy local controllability and $y_{t-1}$ is not too close to $\bd(\Y)$, all points $y_t$ in a ball around $y_{t-1}$ are feasible with an appropriate $x_t$; this enables execution of an  online convex optimization (OCO) algorithm over $\Y$ by playing the action $x_t$ which yields a state update to the target $y_t$ chosen at each iteration, computed via offline non-convex optimization. Here we assume that $D$ is known and can be queried for any inputs, and that disturbances to the state are not present. We allow the dynamics to change over time, potentially as a function of previous actions $x_s$ and losses $f_s$ for $s < t$, provided that $D_t$ can be determined in each round. We use Follow the Regularized Leader ($\ftrl$) as our OCO subroutine \citep{10.1007/11776420_32,Abernethy2008CompetingIT}, yet we note that it may be substituted for any OCO algorithm whose per-round step size is guaranteed to be sufficiently small (such as OGD with a constant learning rate); statements of the $\ftrl$ algorithm and its key properties are provided in Appendix \ref{sec:ftrl}. 
We instantiate $\ftrl$ over a contracted space $\tilde{\Y} \subseteq \Y$, calibrated to ensure that the minimum loss over $\tilde{\Y}$ is close to that for $\Y$, yet where each step of $\ftrl$ lies within the feasible region ensured by (weak) local controllability.

\begin{algorithm}
\caption{Nested Online Convex Optimization ($\oenftrl$).}\label{alg:nonlinear-ftrl}
\begin{algorithmic}
\STATE Let $\psi : \Y \rightarrow \R$ be $\gamma$-strongly convex with $\text{argmin}_{y} \psi(y) = \mathbf{0}$ and $\max_{y, y'} \abs{\psi(y) - \psi(y')} \leq G$
\STATE Let $\eta = (G \gamma)^{1/2} ((1 + \frac{R}{r\rho}) TL^2)^{-1/2} $
\STATE Let $\widetilde{\Y} = \{ y : \frac{1}{1-\delta}y \in \Y\}$ for $\delta = \eta \frac{ L}{r \rho \gamma  }$
\STATE Initialize $\ftrl$ to run for $T$ rounds over $\widetilde{\Y}$ with regularizer $\psi$ and parameter $\eta$
\FOR{$t = 1$ to $T$}
\STATE Let $y^*$ be the point chosen by $\ftrl$
\STATE Use $\texttt{Oracle}(y_{t-1}, y^*)$ to compute $x_t = \argmin_{x} \norm{ D_t(x, y_{t-1}) -  y^*}^2$ 
\STATE Play action $x_t$
\STATE Observe $y_t$ and loss $f_t(y_t)$, update $\ftrl$
\ENDFOR
\end{algorithmic}
\end{algorithm}

\begin{theorem}\label{thm:oenftrl}
For a $\rho$-locally controllable instance $(\X, \Y, D)$ without disturbances and with $D_t$ known at each $t$, the regret of $\oenftrl$  for convex $L$-Lipschitz losses $f_t : \Y \rightarrow \R$ is at most
\begin{align*}
    \Reg_{T}(\oenftrl) \leq&\;   2L\sqrt{(1 + {R}(r\rho)^{-1}) TG\gamma^{-1}} 
\end{align*}
with respect to any state $y^* \in \Y$, with $T$ queries made to a non-convex optimization oracle.
\end{theorem}
The proof for Theorem \ref{thm:oenftrl} is given in Appendix \ref{sec:oenftrl-proof}.

\subsection{Efficient Updates for Action-Linear Dynamics}\label{subsec:ALD}
While $\oenftrl$ requires no assumptions on the dynamics beyond local controllability, there are large classes of dynamics for which the oracle call can be removed. We say that dynamics are {\it action-linear} if $y_x = D(x,y)$ is linear in $x$, for $y_x \in \intr(\Y)$ (and arbitrary for $y_x \in \bd(\Y))$.

\begin{proposition}
For a $\rho$-locally controllable and action-linear instance $(\X, \Y, D)$, 
the per-round optimization problem for $\texttt{\textup{Oracle}}(y_{t-1}, y^*)$ in $\oenftrl$ is convex.
\end{proposition}
\begin{proof}
For $y = y_{t-1} \in \widetilde{\Y} \subseteq \intr(\Y)$, we have $D(x, y) = A_{y} \cdot x + b_y$ for some matrix $A_y$ and vector $b_y$, and so we can solve $x_t = \argmin_{x \in \X} \norm{A_{y} \cdot x + b_y - y^*}^2 $ efficiently.
\end{proof}
The class of action-linear dynamics is quite general, owing to the flexibility permitted by nonlinear parameterizations of $(A_y, b_y)$ in terms of $y$; in Appendix \ref{sec:appendix-action-linear}, we show that local controllability holds for multiple explicit families of instances when appropriate eigenvalue conditions are satisfied. We can further relax this condition to accommodate dynamics where action-linearity holds only {\it locally} in the neighborhood of stabilizing actions (i.e.\ actions $x^*$ where $D(x^*, y) = y$).

\begin{definition}[Locally Action-Linear Dynamics] An instance $(D,\X,\Y)$ is locally action-linear if,
for any $y \in \intr(\Y)$, $x^*$ such that $D(x^*, y) = y$, and $x$ such that $D(x, y) \in \intr(\Y)$, the dynamics are given by $D(x, y) = A_y x + b_y + q_y(x)$, where $A_y$ is a matrix and $b_y$ is a vector, both with norms bounded by some absolute constant, where and $q_y:\X\rightarrow \R^{\dim(\Y)}$ is any function where
$\norm{q_y(x)} \leq  C \norm{ A_y(x - x^*)}^{ 1 + c}$ for some constants $C, c > 0$.
\end{definition}

By this condition, for any $x$ in a sufficiently small neighborhood around $x^*$, the deviation of dynamics (and thus the resulting $y_{t+1}$) from action-linearity vanishes. 
Note that our algorithm always chooses a target $y_t$ will always be near $y_{t-1}$; as such, these deviations from non-action-linearity can be modeled as {\it disturbances} with magnitude strictly less than our per-round step size $\norm{y_{t+1} - y_t}$ (along with universal constant factors). The existence of an efficient implementation follows as a straightforward corollary of Theorem \ref{thm:oenftrlap-body} in Section \ref{subsec:disturbances}, which extends $\oenftrl$ to accommodate bounded adversarial disturbances, as we can then select actions by disregarding the influence of $q_y$ and only considering the local approximation $D(x,y) = A_y x + b_y$ at each state $y$ (assuming that each decomposition between $q_y$ and the action-linear component is known).

\subsection{Adversarial Disturbances}\label{subsec:disturbances}
Our algorithm $\oenftrl$ can be extended to accommodate adversarial disturbances, where the state is updated as $y_{t} = D(x_t, y_{t-1}) + w_t$, with $\{w_t\}$ chosen adversarially. In the weak local controllability case, we show a sharp threshold effect in terms of whether or not $\norm{w_t}$ is allowed to exceed the undisturbed distance from the boundary by a factor of $\frac{\rho}{1 + \rho}$: if disturbances are bounded below this threshold, regret minimization remains feasible with a tight $\Theta(E)$ dependence on the total disturbance magnitude, yet if disturbances may exceed this, no sublinear regret rate is attainable even for a {\it constant} total disturbance magnitude. When $\rho$ is small, an adversary can push us to the boundary faster than we can ``undo'' past disturbances, causing our feasible range to decay.

\begin{theorem}[Bounded Disturbances for Weak Local Controllability]\label{thm:oenftrlap-body}
For any $\rho \in (0, 1]$, suppose that a sequence of adversarial disturbances $w_t$ for a $\rho$-locally controllable instance $(\X, \Y, D)$ satisfies $\sum_{t=1}^T \norm{w_t} \leq E$ and $\norm{w_t} \leq \frac{\rho - \alpha\rho }{1 + \rho}\cdot \pi\parens{ D(x_t, y_{t-1}) }$, for some $\alpha \in \R$. If $\alpha > 0$, there is an algorithm $\oenftrlap$ with regret for convex Lipschitz losses $f_t$ bounded by
\begin{align*}
    \Reg_{T}(\oenftrlap) \leq&\;  O\parens{\sqrt{T\cdot (\alpha\rho)^{-1} } + E},
\end{align*}
and there is an instance where any algorithm $\A$ obtains $\Reg_{T}(\A) = \Omega(E)$.
If $\alpha < 0$, there is an instance 
such that any algorithm $\A$ obtains 
    $\Reg_{T}(\A) \geq  \Omega\parens{T}$
    even when $E = O(1)$.
\end{theorem}
The maximum disturbance bound can be removed when dynamics are strongly locally controllable, as the ensured feasible range of the dynamics does not vanish at the boundary of the state space. For such instances, we can minimize regret (with tight $O(E \cdot \rho^{-1})$ dependence) even if disturbances are only implicitly bounded by the state space diameter (which is at least $\rho$, without loss of generality).
\begin{theorem}[Unbounded Disturbances for Strong Local Controllability]\label{thm:oenftrluap-body}
For any $\rho > 0$ and strongly $\rho$-locally controllable instance $(\X, \Y, D)$ with disturbances $w_t$ satisfying $\sum_{t=1}^T \norm{w_t} \leq E$, there is an algorithm $\oenftrluap$ with regret for convex Lipschitz losses $f_t$ bounded by
\begin{align*}
    \Reg_{T}(\oenftrluap) \leq&\;  O\parens{\sqrt{T} + E \cdot \rho^{-1}},
\end{align*} and there is an instance where any algorithm $\A$ obtains $\Reg_{T}(\A) \geq  \Omega\parens{E \cdot \rho^{-1}}$.
\end{theorem}
In each case, our lower bounds in terms of $E$ hold for the same constants obtained by our algorithms, and our algorithms obtain the stated regret guarantees even when $E$ is not known in advance. We present the algorithms and analysis for each theorem in Appendix \ref{sec:appendix-disturb}; both operate by tracking deviations from an idealized trajectory without disturbances, and calibrating parameters to preserve sufficient reachability margin for applying corrections towards this trajectory in each round. The lower bounds both proceed by considering an instance with a fixed target state $y^*$ and losses which track the distance from $y^*$, along with an adversary whose goal is to maximize this distance by selecting disturbances which push the current state away from $y^*$.

\subsection{Unknown Dynamics}\label{subsec:unknown-dynamics}

Up until this point, we have assumed that the dynamics $D$ can be queried arbitrarily in each round. While this has required minimal assumptions on $D$ beyond local controllability, accommodation of unknown dynamics is often desired in online control \citep{cassel2022efficient,minasyan2022online} and for several of our applications \citep{DBLP:journals/corr/RothUW15,agarwal2023online}. Here we give conditions under which regret minimization can be implemented without advance knowledge of $D$ by an algorithm $\probingoco$, which maintains continuously-updating local linear approximations of $D$ near $y_{t}$ across rounds. Crucially, we assume that $D$ is time-invariant and locally action-linear with sufficiently small Lipschitz parameters, and that for the initial state $y_0$ some {\it near-stabilizing} action $x_1$ is known, i.e.\ $\norm{ D(x_1, y_0) - y_0} \leq \epsilon$, for some $\epsilon = o(\sqrt{T})$.

\begin{theorem}\label{thm:unknown-dynamics} For any $\rho$-locally controllable and time-invariant instance $(D, \X, \Y)$ which satisfies local action-linearity and appropriate Lipschitz conditions, there is an algorithm $\probingoco$ with $\Reg_T(\probingoco) \leq O(\sqrt{T})$ 
for convex Lipschitz losses $f_t$ 
and unknown dynamics $D$, provided that at $t=1$ we are given some $x_1$ such that $\norm{ D(x_1, y_0) - y_0} = o(\sqrt{T})$.
\end{theorem}

We state $\probingoco$ and prove Theorem \ref{thm:unknown-dynamics} in Appendix \ref{sec:appendix-unknown}, along with additional details on the regularity and near-stability assumptions. The crux of our analysis, beyond that from our previous results, hinges on being able to maintain and update local linear approximations of $D$ throughout our optimization which are sufficiently accurate to allow us to discard the effects of both learned representation errors and action non-linearity from $q_y(x)$ as bounded disturbances. We implement each update from our nested regret minimization algorithm as a series of $O(\dim(\X))$ steps involving small near-orthogonal perturbations to our targets $y_t$, which we then use to update our local estimate for $D$.

\subsection{Bandit Feedback}
\label{subsec:bandit}

We can extend our approach from $\oenftrl$ to accommodate bandit feedback for convex losses by replacing $\ftrl$ with the $\fkm$ algorithm \citep{FKM04} and appropriately recalibrating parameters. $\fkm$ obtains ${O}(T^{3/4})$ regret, which is the best currently-known bound for bandit convex optimization without additional assumptions (e.g. strong convexity), and we obtain an analogous bound here for nested optimization. We note that this extension to bandit feedback can again be applied for any algorithm with a small per-round step-size bound, though this property does not hold for algorithms which sample from larger sets to reduce variance of gradient estimators (e.g.\ those from \citet{Abernethy2008CompetingIT,NIPS2014_c399862d}). 
\begin{theorem}\label{thm:bandit}
For any $\rho$-locally controllable instance $(D,\X,\Y)$, there is an oracle-efficient algorithm $\nestedbco$ with expected regret bounded by
\begin{align*}
    \Reg_{T}(\nestedbco) =&\;  O \parens{nRLT^{3/4} (r\rho)^{-1}}
\end{align*}
for $L$-Lipschitz convex losses $f_t$ under bandit feedback.
\end{theorem}
We present the $\nestedbco$ algorithm and prove Theorem \ref{thm:bandit} in Appendix \ref{sec:appendix-bandit}.

\section{Applications for Online Stackelberg Optimization}
\label{sec:apps}

We give several applications of our framework to online Stackelberg problems involving strategic or adaptive agents, each cast as an instance of online control with nonlinear dynamics where local controllability holds, and where our objectives are well-approximated by convex surrogate losses only over the state. 
Each application extends prior work by either allowing for more relaxed assumptions, unifying distinct problem instances, or giving a novel formulation to account for dynamic and adversarial behavior; analysis and comparison to related work is contained in Appendices \ref{sec:appendix-pp}-\ref{sec:appendix-steering}.

\subsection{Online Performative Prediction}
\label{subsec:app-pp}

Performative Prediction was introduced by \cite{DBLP:journals/corr/abs-2002-06673} to capture settings in which the data distribution may shift as a function of the classifier itself.
We consider the online formulation of Performative Prediction introduced in \cite{kumar2022online} as an instance of online convex optimization with unbounded memory, which we extend to accommodate a {\em stateful} variant of the problem (as in \cite{brown2022performative}) in which the update to the distribution is a function of both the classifier and the current distribution itself. 
Let $\X \subseteq \R^n$ denote our space of classifiers, and let $p_0$ be the initial distribution over $\R^n$. When a classifier $x_t$ is deployed, the distribution is updated to 
\begin{align*}
    p_{t} =&\; (1 - \theta)p_{t-1} + \theta \D(x_t, y_{t-1})
\end{align*}
where $\D(x_t, y) = A(x_t, y_{t-1}) + \xi$, for a random variable $\xi \in \R^n$ with mean $\mu$ and covariance $\Sigma$, and with $y_t = A(x_t, y_{t-1})$, where $A$ satisfies $\rho$-local controllability for some $\rho > 0$ and appropriate smoothness notions. We also assume there is some linear $s : \X \rightarrow \Y$ such that $A(x, y) = s(x)$ if $y = s(x)$. We then receive loss $\tilde{f}_t(x_t, p_t) = \E_{z\sim p_t}[f_t(x_t, z)]$, where each $f_t$ is convex and Lipschitz.

This generalizes the model of \cite{kumar2022online}, in which $A(x, y) = A \in \R^{n \times n}$ is taken to be a fixed matrix; there, $\rho$-local controllability is satisfied for some $\rho > 0$ provided that $A$ is nonsingular. Their aim is to compete with the best fixed classifier by running regret minimization over $\X$. Here we run $\oenftrl$ over $\Y$, taken over the range of $s$, which allows us to compete against the best fixed classifier as well by the properties of $s$; while the classifiers $x_t$ we play will generally not result in stabilizing points of $A$, their excess loss compared to each $s^{-1}(y_t)$ is bounded. 

\begin{theorem}[Regret Minimization for Performative Prediction]\label{thm:pp-reg-body} 
For any $\theta > 0$, the dynamics for Online Performative Prediction are $\rho$-locally controllable, and $\oenftrl$ obtains regret $O(\sqrt{T(\rho^{-1} + \theta^{-1})})$  with respect to the best fixed classifier.
\end{theorem}

\subsection{Adaptive Recommendations}
\label{subsec:app-recs}

Online interactions with economic agents of various types are ubiquitous, and the resulting control problems tend to be manifestly nonlinear; here we treat two diverse examples from this space.  The Adaptive Recommendations problem, as introduced by \cite{AB22}, is about providing menu recommendations repeatedly to an agent, whose choice distribution is a function of their past selections, while the controller's reward in each round depends on adversarial losses over the choice. In each round $t \in [T]$, we show the agent a (possibly randomized) menu $K_t$ containing $k$ (out of $n$) items, and the agent's instantaneous choice distribution conditioned on seeing $K_t$ is
\begin{align*}
    p_t(i ; K_t, v_{t-1}) =&\;  \begin{cases} 
    \frac{s_{i}(v_{t-1})}{\sum_{j \in K_t} s_j(v_{t-1})} & i \in K_t \\
      0 & i \notin K_t
   \end{cases}
\end{align*}
where each $s_i : \Delta(n) \rightarrow [\lambda, 1]$ is the agent's {\it preference scoring function} for item $i$, for some $\lambda > 0$, taking as input the agent's {\it memory vector} $v \in \Delta(n)$. The memory vector updates each round as
\begin{align*}
    v_{t} = (1 - \theta_t)v_{t-1} + \theta_t p_t,
\end{align*}
where $\theta_t \in [\theta,1]$ for $\theta > 0$ is a possibly time-dependent update speed, and we receive loss $f_t(p_t)$, where each $f_t$ is convex and $L$-Lipschitz. Note that the set of feasible choice distributions when considering all menu distributions $x_t \in \Delta({n \choose k})$ depends on the memory vector $v_t$.
The regret benchmark considered by \cite{AB22} is the intersection of all such sets, denoted the ``everywhere instantaneously-realizable distribution'' set $\EIRD = \cap_{v \in \Delta} \IRD(v)$, where $\IRD(v)$ is the ``instantaneously realizable distribution'' set for $v$, given as the convex hull of the choice distributions $p(K_t)$ resulting from each menu $K_t \in [{n\choose k}]$ when $v$ is the memory vector. It is shown that the set is non-empty when $\lambda$ is not too small, and  algorithms which  minimize regret with respect to any distribution in $\EIRD$ are given in \cite{AB22} and \cite{agarwal2023online} under varying assumptions regarding the scoring functions and update speed.

While the prior work considers a bandit version of the problem with unknown dynamics, here we consider a full-feedback deterministic variant of the problem for simplicity, which further allows us to circumvent barriers posed by uncertainty \cite{AB22,agarwal2023online} and relax structural assumptions (e.g.\ on $\theta_t$ or $s_i$). We can cast this as an instance of our framework by taking $\X = \Delta({n \choose k})$ and $\Y = \EIRD$, where $D$ expresses updates to the memory vector. 
We assume $v_0 = \mathbf{u}_n$, and we reparameterize to run our algorithm over $\Delta(n)$.
We optimize surrogate losses $f^*_t(v_t)$, and bound excess regret from $f_t(p_t)$.

\begin{theorem}[Regret Minimization over $\EIRD$]\label{thm:ar-eird} For $\lambda > \frac{k-1}{n-1}$, the dynamics for Adaptive Recommendations over $\EIRD$ are $\theta$-locally controllable, and $\oenftrl$ obtains regret $O(\sqrt{T\theta^{-1}})$.
\end{theorem} In \cite{agarwal2023online}, a property for scoring functions is considered which enables regret minimization over a potentially much larger set of distributions than $\EIRD$.
A scoring function $s_i : \Delta(n) \rightarrow [\frac{\lambda}{\sigma}, 1]$ is said to be $(\sigma, \lambda)$-{\it scale-bounded} for $\sigma > 1$ if, for all $v\in \Delta(n)$, we have that
\begin{align*}
    \sigma^{-1} ((1 - \lambda)v_i + \lambda) \leq s_i(v) \leq \sigma ((1 - \lambda)v_i + \lambda).
\end{align*}
 The set considered is the $\phi$-{\it smoothed simplex} $\Delta^{\phi}(n) = \{(1 - \phi)v + \phi \mathbf{u}_n : v \in \Delta(n) \}$, for $\phi = \Theta(k\lambda\sigma^2)$, where it is shown that $\IRD(v)$ contains a ball around $v$ for $v \in \Delta^{\phi}(n)$.
We take $\Y = \Delta^{\phi}(n)$, which satisfies local controllability, and optimize over $f_t^*(v_t)$ with $\oenftrl$.

\begin{theorem}[Regret Minimization over $\Delta^{\phi}(n)$]\label{thm:ar-ss}
For $(\sigma, \lambda)$-{scale-bounded} scoring functions $s_i$, 
for any $\lambda > 0$ and $\sigma > 1$, 
the dynamics for Adaptive Recommendations over $\Delta^{\phi}(n)$ are $\Omega(\theta\lambda\phi)$-locally controllable, and $\oenftrl$ obtains regret $O(\sqrt{T(\theta\lambda\phi)^{-1}})$.
\end{theorem}

\subsection{Adaptive Pricing}
\label{subsec:app-pricing}

Here we consider an Adaptive Pricing problem for real-valued goods, formulated as a dynamic extension of the setting of  \cite{DBLP:journals/corr/RothUW15} where purchase history and consumption affect demand. In each round we set per-unit price vectors $p_t \in \R_{+}^{n}$, and an agent buys some bundle of goods $x_t \in \R_{+}^n$, which results in us obtaining a reward $\langle p_t, x_t \rangle - c_t(x_t)$, where our production cost function $c_t$ at each round is convex and $L_c$-Lipschitz, and may be chosen adversarially.  

Departing from \cite{DBLP:journals/corr/RothUW15}, we consider an agent who maintains goods reserves $y_{t-1} \in \R_{\geq 0}^n$ and consumes an adversarially chosen fraction $\theta_t \in [\theta, 1]$ of every good's reserve  at each round (for some $\theta > 0$). The agent then chooses a bundle $x_t$ to maximize their utility $g(p_t, x_t, y_t) =  v(y_t) - \langle p_t, x_t \rangle$, where $y_t = (1 - \theta_t)y_{t-1} + x_t$ is their updated reserve bundle.
We make several regularity assumptions on the agent's valuation function $v : \R_{+}^n \rightarrow \R_{+}$, all of which are satisfied by several classically studied utility families (which we discuss in Appendix \ref{subsec:app-pricing}). Notably, we assume that $v$
is strictly concave and increasing, 
and homogeneous; the range is bounded under rationality. 

Our aim will be to set prices which allow us to compete with the best {\it stable reserve policy}, e.g.\ against any pricing policy where the agent maintains the same reserve bundle $y_t = y^*$ at each round for some $y^*$ regardless of $\theta_t$. We take an appropriate convex set of such bundles as our state space, for which we show that local controllability holds.
Observe that to induce a purchase of $x_t = \theta_t y_{t-1}$, it suffices to set prices $p_t = \nabla v(y_{t-1})$, as we then have that $\nabla_{x_t} (v((1 - \theta_t)y_{t-1} + x_t)  - \langle p_t, x_t \rangle) = \mathbf{0}$. By homogeneity of $v$, we also have that $\langle \nabla v(y_t), \theta_{t}y_t \rangle =  \theta_{t}k \cdot v(y_t)$ for some $k$, and we show that optimization via the concave surrogate rewards
\begin{align*}
    f^*_t(y_t) =&\; \theta_t k \cdot v(y_t) - c_t(\theta_t y_t )
\end{align*}
will closely track our true rewards $f_t(p_t, x_t) = \langle p_t, x_t \rangle - c_t(x_t)$. While neither our true nor surrogate rewards will be Lipschitz, we extend $\oenftrl$ to obtain sublinear regret over Hölder continuous losses by appropriately calibrating our step size  (which may be of independent interest).
\begin{theorem}[Regret Minimization over Stable Reserve Policies] \label{thm:pricing-body}
For any $\theta > 0$, the dynamics for Adaptive Pricing can are $\theta$-locally controllable, and $\oenftrl$ obtains regret $o(T \theta^{-1})$ with respect to the best stable reserve policy.
\end{theorem}

\subsection{Steering Learners in Online Games}
\label{subsec:app-steering}
A recent line of work \citep{deng2019strategizing,mansour2022strategizing,brown2023learning} explores maximizing rewards in a repeated game against a no-regret learner, and \cite{anagnostides2023convergence} study of no-regret dynamics in time-varying games. We consider these questions in unison, and aim to optimize reward against a no-regret learner for game matrices chosen adversarially and online.

Consider adversarial sequences of two-player $m \times n$ bimatrix games $(A_t,B_t)$, where $m > n$; we assume that the convex hull of the rows of each $B_t$ contains the unit ball. 
As Player A, we choose strategies $x_t \in \Delta(m)$ each round to maximize our reward against Player B, who chooses their strategies $y_t \in \Delta(n)$ according to a no-regret algorithm (in particular, online projected gradient descent). The game $(A_t, B_t)$ is only revealed after both players have chosen strategies for round $t$. 
Our aim here is to illustrate the feasibility of {\it steering} the opponent's trajectory, and so we consider games where Player A's reward is predominantly a function only of Player B's actions.
We assume that $\norm{x A_{t} - x {A^*_t}} \leq \delta_t$ for any $x \in \Delta(m)$, where each $A^*_t$ is a matrix with identical rows,
and that per-round changes to $B_t$ are bounded, with $\norm{x B_{t} - x B_{t-1}} \leq \epsilon_t$ for any $x \in \Delta(m)$. We measure the regret of an algorithm $\A$ with respect to {\it any} profile $(x, y) \in \Delta(m) \times \Delta(n)$, where
\begin{align*}
    \Reg_{T}(\A) =&\; \max_{(x, y) \in \Delta(m) \times \Delta(n)}  \sum_{t=1}^T x A_t y -  x_t A_t y_t.
\end{align*}
When Player B plays 
$\opgd$
with step size $\theta = \Theta(T^{-1/2})$, their strategy updates each round as
\begin{align*}
    y_{t+1} =&\; \Pi_{\Delta(n)} \parens{ y_{t}  + \theta (x_{t} B_{t} ) },
\end{align*}
with $y_1 = \mathbf{u}_n$, and yields regret $O(\sqrt{T})$ for Player B with respect to any $y \in \Delta(n)$ for the loss sequence $\{x_tB_t : t \in [T]\}$.
To cast this in our framework, we consider $\Delta(n) = \Y$ as our state space, where we select actions $x_{t-1}$ to induce desired updates to $y_t$ and optimize over the surrogate losses $\{\mathbf{u}_m A^*_t y_t : t \in [T]\}$. 
While we do not see $B_t$ prior to choosing each $x_t$, we view our update errors from instead selecting an action in terms of the dynamics resulting from $B_{t-1}$ as adversarial disturbances and run $\oenftrluap$, as the dynamics are strongly locally controllable.

\begin{theorem}[Regret Minimization in Online Games]\label{thm:games-body} For $\theta  = \Theta(T^{-1/2})$, repeated play against $\opgd$ in online 
$m \times n$
games can be cast as a $\theta$-strongly locally controllable instance of online control with nonlinear dynamics, for which $\oenftrluap$ obtains regret $O(\sqrt{T} + \sum_t (\delta_t + \epsilon_t))$.
\end{theorem}

\bibliographystyle{plainnat}
\bibliography{ref}
\clearpage
\appendix

\section{Omitted Proofs for Section \ref{sec:model}}
\label{sec:subopt-linear}

\begin{proof}\textbf{of Proposition \ref{prop:local-controllability}.}
Without loss of generality, assume $\alpha \leq \beta/2$ and that $T$ is even. Let $f_t = \norm{y_t - y}$ for each $t$. 
Consider any round $t$ where $y_{t-1} \in B_{\alpha}(y)$; then, for all actions $x_t$, we have that $y_t \notin \B_{\alpha}(y)$, as $\B_{\alpha}(y) \subseteq \B_{\beta}(y_{t-1})$; as such, we incur loss $f_t(y_t) \geq \alpha$ in round $t$. 
Now suppose $y_{t-1} \notin B_{\alpha}(y)$; then, we must have incurred loss at least $f_{t-1}(y_{t-1}) \geq \alpha$ in round $t-1$. 
As losses are non-negative, our total loss is at least $\alpha T/2$, as loss $\alpha$ is incurred at least every other round; given that the best fixed state $y^* = y$ incurs total loss $0$, we have that $\Reg_{\A}(T) = \Omega(T)$ for any algorithm $\A$.
\end{proof}

\begin{proof}\textbf{of Proposition \ref{prop:linear-policies}.}
We begin by observing that for instances $(\X, \Y, D)$, the class of state-targeting policies contains a policy which obtains the reward of the best fixed state up to $O(\sqrt{T {\rho^{-1}}})$, for sufficiently large $T$. Consider the set $\hat{\Y} = \{y^* \in \Y : \pi(y^*) \geq (T \rho  )^{-1/2} \}$. Note that the reward of any $y \in \Y$ is matched by some $y^* \in \hat{\Y}$ up to $O(\sqrt{T \rho^{-1}})$ for any fixed inner radius $r$, outer radius $R$, and  Lipschitz  constant $L$. 
For any such $y^*$, note that under the policy $P_{y^*}$ when starting at $y_0 = 0$, the distance between $y_t$ and $y^*$ in each round $t$ is updated to at most:
\begin{align*}
    \norm{y_t - y^*} \leq&\; \max \parens{0, \rho\cdot \pi(y_{t-1})}.
\end{align*}
It is straightforward to see that $\hat{\Y}$ is convex, and so our state $y_t$ will never leave $\hat{\Y}$ on its path to $y^*$; as such, we reach $y^*$ within $O(\sqrt{T \rho^{-1}})$ rounds, after which point our reward exactly tracks that of $y^*$. For some $y^* \in \hat{\Y}$, this yields a regret for $P_{y^*}$ of at most $O(\sqrt{T \rho^{-1}})$ to the best fixed state in $\Y$.

Next, consider an instance where $\X$ and $\Y$ are both the unit ball in $\R^n$. With $y_0 = 0$, let the dynamics be given by
\begin{align*}
    y_{t} =&\; \Pi_{\Y} \parens{ y_{t-1} +   x_t}. 
\end{align*}
Observe that this satisfies $\rho$-local controllability for any $\rho \leq 1$, as a ball of radius $\pi(y_{t-1})$ is always feasible around $y_{t-1}$.
Let each loss $f_t = \norm{y - p}^2$, for some $p \neq 0$. Immediately we can see that any matrix policy $K \in \mathcal{P}_{\mathcal{K}}$ has regret $\Omega(T)$, 
as the action $x_t = 0$ will be played in each round. 
\end{proof}

\section{Follow the Regularized Leader}\label{sec:ftrl}
Here we state the $\ftrl$ algorithm and several of its key properties; see e.g.\ \cite{hazan2021introduction} for proofs of Propositions \ref{prop:ftrl} and \ref{prop:ftrl-step}.

\begin{algorithm}
\caption{Follow the Regularized Leader ($\ftrl$)}\label{alg:ftrl}
\begin{algorithmic}
\STATE Choose a time horizon $T$, step size $\eta$, and $\gamma$-strongly convex regularizer $\psi : \Y \rightarrow \R$
\STATE Let $y_1 = \text{argmin}_{y \in \Y}~ \psi(y)$
\FOR{$t = 1$ to $T$}
\STATE Play $y_t$ and observe loss $f_t(y_t)$
\STATE Set $\nabla_t = \nabla f_t(y_t)$
\STATE Set $y_{t+1} = \text{argmin}_{y \in \Y} \parens{ \eta \cdot \sum_{s=1}^t \nabla_s^{\top} y  + \psi(y) }$
\ENDFOR
\end{algorithmic}
\end{algorithm}
\begin{proposition}\label{prop:ftrl} For a $\gamma$-strongly convex regularizer $\psi : \Y \rightarrow \R$ where $\abs{\psi(y) - \psi(y')} \leq G$ for all $y, y' \in \Y$, and for convex $L$-Lipschitz losses $f_1,\ldots, f_T$, the regret of $\ftrl$ is bounded by
\begin{align*}
    \Reg_{T}(\textup{$\ftrl$}) \leq&\; \eta \frac{TL^2}{\gamma} + \frac{G}{\eta}.
\end{align*}
\end{proposition}
\begin{proposition}\label{prop:ftrl-step}
    Any pair of points $y_t$ and $y_{t+1}$ chosen by $\ftrl$ satisfies
$\norm{y_{t+1} - y_t} \leq \eta \frac{L}{\gamma}$.
\end{proposition}

\section{Analysis for $\oenftrl$}\label{sec:oenftrl-proof}
\begin{proof}\textbf{of Theorem \ref{thm:oenftrl}.} First we show that any point chosen by $\ftrl$ will be feasible under  local controllability, by induction. It is straightforward to see that $\tilde{\Y}$ is convex and $\tilde{\Y} \subseteq \Y$; further, any $y \in \tilde{\Y}$ is bounded away from $\bd(\Y)$.
By the definition of $\tilde{\Y}$, we have that $y = (1-\delta)y'$ for some $y' \in \Y$. Recall that $\B_{r}(\mathbf{0}) \subseteq \Y$, and note that $\B_{\delta r}(y) = \{y + \delta \hat{y} : \hat{y} \in \B_{r}(\mathbf{0}) \}$. Let $y''$ be any point in $\B_{r}(\mathbf{0})$. By convexity of $\Y$, we then have that any point $(1 - \delta)y' + \delta y''$ lies in $\Y$, and so for any $y \in \tilde{\Y}$ we have that $\B_{r\delta}(y) \subseteq \Y$. 
Each $y_{t-1}$ lies in $\tilde{\Y}$, and so we have that $\pi(y_{t-1}) \geq r \delta$; 
as such, any point $y_{t}$ in  $\B_{r\delta\rho}(y_{t-1}) \subseteq \B_{\rho \cdot \pi(y_{t-1})}(y_{t-1})$ is feasible. 
Given that $\eta \frac{L}{\gamma} \leq r  \delta \rho$, by Proposition \ref{prop:ftrl-step} we have that $y_{t} \in \B_{r\delta\rho}(y_{t-1})$ in each round for the chosen point. Each action will be selected by solving for
\begin{align*}
    \argmin_{x_t \in \X} \norm{D(x_t, y_{t-1}) - y^*}^2
\end{align*}
via a call to $\texttt{Oracle}(y_{t-1}, y^*)$. Each call is guaranteed to have a solution which achieves an objective of 0 where $D(x_t, y_{t-1}) = y^*$  for some $y^* \in \B_{\rho \cdot \pi(y_{t-1})}(y_{t-1})$ by local controllability, yielding an exact state update to $y_t = y^*$ as we assume $\texttt{Oracle}$ can solve arbitrary non-convex minimization problems. To bound the regret, first note that for any $y^* \in {\Y}$, we have 
\begin{align*}
    \sum_{t=1}^T f_t(y_t)  \leq&\; \eta \frac{TL^2}{\gamma} + \frac{G}{\eta} + \sum_{t=1}^T f_t((1 - \delta)y^*)
\end{align*}
by Proposition \ref{prop:ftrl}, as $(1 - \delta)y^* \in \tilde{\Y}$ for any $y^* \in \Y$. Then, observe that for any $y^* \in \Y$, we have that
\begin{align*}
    \sum_{t=1}^T f_t((1-\delta)y^*) \leq&\; \sum_{t=1}^T \parens{ f_t(y^*) + L\norm{\delta y^*} } \\
    \leq&\; \sum_{t=1}^T \parens{ f_t(y^*) + \delta LR }. 
\end{align*}
Combining the previous claims, we have that
\begin{align*}
    \sum_{t=1}^T f_t(y_t) - f_t(y^*) \leq&\; \delta TLR +  \eta \frac{TL^2}{\gamma} + \frac{G}{\eta} \\
    =&\;  \eta \parens{1 + \frac{R}{r\rho}}  \frac{TL^2}{\gamma} + \frac{G}{\eta} \\
    =&\; 2 \sqrt{\frac{(1 + \frac{R}{r\rho}) TGL^2}{\gamma}}
\end{align*}
upon setting $\delta = \eta \frac{ L}{r \rho \gamma  }$ and $\eta = \sqrt{ \frac{ G \gamma }{(1 + \frac{R}{r\rho}) TL^2} }$, which yields the theorem.
\end{proof}

\section{Examples and Analysis for Action-Linear Dynamics}\label{sec:appendix-action-linear}

As a simple yet general example of dynamics which are both action-linear and locally controllable, consider update rules in which a step is taken by applying a nonsingular matrix transformation to the action, where the matrix can be parameterized by the state, with projection back into $\Y$ if necessary. 
\begin{example}\label{ex:1}
Let both $\X$ and $\Y$ be given by the unit ball $\B_{1}(\mathbf{0})$ in $\R^n$.
For any fixed $y$, let the updates from $D(x, y)$ be given by
\begin{align*}
D(x, y) =&\; \Pi_{\Y}\parens{ y + A_y \cdot x},
\end{align*}
where each $A_y$ is a square matrix with 
minimum absolute eigenvalue $\abs{\lambda_n(A_y)} \geq \pi(y) \cdot \rho$ for some $\rho > 0$. Then, the instance $(\X, \Y, D)$ is action-linear and satisfies $\rho$-local controllability.   
\end{example}
\begin{proof}\textbf{for Example \ref{ex:1}.} It is straightforward to see that $D(x,y)$ is action-linear. To show $\rho$-local controllability, let $y^*$ be any point in $\B_{\rho \cdot \pi(y)}(y)$. It suffices to show that there is some $x^* \in \X$ such that $A_y \cdot x^* = y^* - y$. As $A_y$ is non-singular, we can solve for $x^* = A_y^{-1}(y^* - y)$, where $\norm{ y^* - y } \leq \rho \cdot \pi(y)$ and $\abs{\lambda_1(A_y^{-1})} \leq \frac{1}{\rho \cdot \pi(y)}$, and so we have that $x^* \in \B_1(\mathbf{0}) = \X$.
\end{proof}
We can also extend this to include state-parameterized generalizations of any linear system governed by nonsingular matrices over a bounded-radius state space (for a sufficiently large action space).
\begin{example}\label{ex:2}
Let  $\Y$ be given by the radius-$R$ ball $\B_{R}(\mathbf{0})$ in $\R^n$, and let $\X =\B_{cR}(\mathbf{0})$. 
For any fixed $y$, let the updates from $D(x, y)$ be given by
\begin{align*}
D(x, y) =&\; \Pi_{\Y}\parens{ K_{y} \cdot y + A_y \cdot x},
\end{align*}
where both $K_y$ and $A_y$ are square matrices.
For any $y$, let $M_y = K_y - I$, and suppose we take $c$ large enough such that 
$c \cdot \abs{\lambda_n(A_y)} \geq \abs{\lambda_1(M_y)} + \pi(y) \cdot \rho$
for some $\rho > 0$.
Then, the instance $(\X, \Y, D)$ is action-linear and satisfies $\rho$-local controllability.   
\end{example}
\begin{proof}\textbf{for Example \ref{ex:2}.} Here, again it is evident that $D(x,y)$ is action-linear, and so it suffices to show that there is some $x^* \in  \X$ such that
\begin{align*}
    K_y \cdot y + A_y \cdot x^* =&\; y + M_y \cdot y + A_y \cdot x^* \\
    =&\; y^*
\end{align*}
for any $y^*$ in $\B_{\rho \cdot \pi(y)}(y)$. As in the proof for Example \ref{ex:1}, we have that $\norm{M_y \cdot y} \leq R \cdot \abs{\lambda_1(M_y)}$, and  for large enough $c$ there is some $x^*$ such that $A_y \cdot x^* = \hat{y}$ for any $\hat{y}$ where $\norm{\hat{y}} \leq R\cdot \abs{\lambda_1(M_y)} + \pi(y) \cdot \rho$. Thus, any point $y^* \in \B_{R \cdot \abs{\lambda_1(M_y)} + \pi(y) \cdot \rho}(y + M_y \cdot y)$ is feasible by some $x^*$, which contains the ball $\B_{\pi(y) \cdot \rho}(y)$.
\end{proof}

\section{Algorithms for Adversarial Disturbances}
\label{sec:appendix-disturb}
\subsection{$\oenftrlap$ and Proofs for Theorem \ref{thm:oenftrlap-body}}

We show that it is possible simulate $\oenftrl$ over the undisturbed states $\hat{y}_t$ under the assumption that the dynamics are in $\alpha\rho$-locally controllable for some $\alpha \in (0,1)$ while retaining sufficient range in the feasible region around $y_t$ to correct for the disturbance $w_{t-1}$ from the previous round. Here, the oracle call for computing $x_t$ in each round is updated to consider the true state $y_{t-1}$.

\begin{algorithm}
\caption{$\oenftrl$ with Adversarial Disturbances ($\oenftrlap$).}\label{alg:nonlinear-ftrl-ap}
\begin{algorithmic}
\STATE Initialize $\oenftrl$ for $T$ rounds over $(\X, \Y, D)$ for $\alpha \rho$-locally controllable dynamics
\FOR{$t = 1$ to $T$}
\STATE Let $\hat{y}_t$ be the target state chosen by $\oenftrl$
\STATE Use $\texttt{Oracle}(y_{t-1}, \hat{y}_t)$ to compute  $x_t = \argmin_{x\in \X} \norm{ D(x, y_{t-1}) - \hat{y}_{t}}^2$ 
\STATE Play action $x_t$.
\STATE Observe disturbed state $y_t = \hat{y}_t + w_t$ and loss $f_t(y_t)$.
\STATE Update $\oenftrl$ with state $\hat{y}_t$ and loss $f_t(\hat{y}_t)$.
\ENDFOR
\end{algorithmic}
\end{algorithm}

Theorem \ref{thm:oenftrlap-body} follows directly from Theorems \ref{thm:oenftrlap-regret}, \ref{thm:oenftrlap-E-bound}, and \ref{thm:oenftrlap-lower-bound}. Intuitively, when the per-round disturbance magnitude is at most $\frac{\rho - \alpha \rho }{ 1 + \rho }\cdot \pi\parens{ D(x_t, y_{t-1}) }$, one can calibrate $\oenftrl$ for the case of $\alpha \rho$-locally controllable dynamics and maintain sufficient ``slack'' to correct for the previous round's disturbance in every round. When disturbances exceed $\frac{\rho }{ 1 + \rho }\cdot \pi\parens{ D(x_t, y_{t-1}) }$, an adversary can continually push the state towards the boundary of $\Y$, which may require vanishing disturbance magnitude as rounds progress due to the limited range promised by local controllability near the boundary.

\begin{theorem}\label{thm:oenftrlap-regret}
For a $\rho$-locally controllable instance $(\X, \Y, D)$ with convex 
losses $f_t : \Y \rightarrow \R$ and adversarial disturbances $w_t$ where $\norm{w_t} \leq \frac{\rho - \alpha \rho }{ 1 + \rho }\cdot \pi\parens{ D(x_t, y_{t-1}) }$ and  $\sum_{t=1}^T \norm{w_t} \leq E$, the regret of $\oenftrlap$ with respect to the reward of any state is bounded by
\begin{align*}
    \Reg_{T}(\oenftrlap) \leq&\;  O\parens{\sqrt{T\cdot (\alpha \rho)^{-1} } + E },
\end{align*}
with $T$ queries made to an oracle for non-convex optimization.
\end{theorem}
\begin{proof}
We show by induction that each call to $\texttt{Oracle}(y_{t-1}, \hat{y}_t)$ yields a feasible action $x_t$ satisfying $\hat{y}_t = D(x_t, y_{t-1})$. This is immediate for $t=1$, and suppose this holds up to some round $t-1$, where we have that $y_{t-1} = \hat{y}_{t-1} + w_{t-1}$.
Given that $\oenftrl$ selects actions under $\alpha\rho$-local controllability, we can bound
\begin{align*}
    \norm{\hat{y}_{t} -\hat{y}_{t-1} } \leq&\; \alpha\rho  \cdot \pi(\hat{y}_{t-1}). 
\end{align*}
Further, the magnitude of the disturbance $w_{t-1}$ is bounded by
\begin{align*}
    \norm{w_{t-1}} \leq&\; \frac{\rho - \alpha \rho }{ 1 + \rho }\cdot \pi(\hat{y}_{t-1}), 
\end{align*}
yielding  that
\begin{align*}
    \norm{\hat{y}_t - y_{t-1}} \leq&\; \norm{\hat{y}_t - \hat{y}_{t-1} - w_{t-1}} \\ 
    \leq&\; \parens{\alpha \rho + \frac{\rho - \alpha \rho }{ 1 + \rho } } \cdot \pi(\hat{y}_{t-1}) \tag{$y_{t-1} = w_{t-1} + \hat{y}_{t-1}$}.
\end{align*}
As such, we have that 
\begin{align*}
     \rho \cdot \pi({y}_{t-1}) \geq&\; \rho \parens{1 - \frac{\rho - \alpha \rho }{ 1 + \rho }} \cdot \pi(\hat{y}_{t-1}) \\
     =&\; \rho \parens{\alpha  + \frac{1- \alpha  }{ 1 + \rho } } \cdot \pi(\hat{y}_{t-1}), 
\end{align*}
and so by $\rho$-local controllability some feasible action $x_t$ exists, as $\hat{y}_t$ lies in $\B_{\rho \cdot \pi(y_{t-1})}$. The regret bound for $\oenftrl$ holds over the states $\hat{y}_t$, and so we can bound the total regret of $\oenftrlap$  with respect to any $y^* \in \Y$ as:
\begin{align*}
    \sum_{t=1}^T f_t(y_t) - f_t(y^*) \leq&\; \sum_{t=1}^T f_t(\hat{y}_t) - f_t(y^*) + L\norm{y_t - \hat{y}_t} \\
    \leq&\; \Reg_{T}(\textup{OEN-FTRL}) + L \sum_{t=1}^T \norm{w_t} \tag{Thm.\ \ref{thm:oenftrl}} \\
    \leq&\; 2 \sqrt{\frac{(1 + \frac{R}{r\alpha\rho}) TGL^2}{\gamma}} + LE.
\end{align*}
\end{proof}

We show that the dependence on $E$ is tight up to the constant. Note that we 
we can obtain regret $O(\sqrt{T\cdot (\alpha\rho)^{-1}}) + LE$ in the following instance via $\oenftrlap$.

\begin{theorem}[Regret Lower Bound for Bounded Disturbances] \label{thm:oenftrlap-E-bound}Suppose for any  $\alpha > 0$ and $\rho \in (0,1]$ an adversary can choose $w_t$ with $\norm{w_t} \leq \frac{\rho - \alpha \rho }{ 1 + \rho } \cdot \pi\parens{ D(x_t, y_{t-1}) }$, where  $\sum_{t=1}^T \norm{w_t} = E$ for any $E$.
There is a $\rho$-locally controllable instance $(\X, \Y, D)$ with $L$-Lipschitz convex losses $f_t$ such that any algorithm $\A$ obtains regret $\Reg_{T}(\A) \geq  \max(LE, \frac{\rho - \alpha \rho }{ 1 + \rho }TL)$. 
\end{theorem}
\begin{proof}
Consider any norm $\norm{\cdot}$ over $\R^n$. 
Let $\Y$ be the unit ball $B_1(\mathbf{0})$, and let each $f_t(y_t) = L\norm{y_t}$. Consider any action space $\X$ and dynamics $D$
where $\rho$-local controllability exactly characterizes the range of $D$, i.e.\ for any $y$ and $y'$, there is some $x$ such that $D(x, y) = y'$ if and only if $y' \in \B_{\rho \cdot \pi(y)}(x, y)$.

First, note that $\pi(y) = 1 - \norm{y}$ for any $y \in \Y$. In each round $t$ , suppose an algorithm plays an action $x_t$ at state $y_{t-1}$ which yields an target undisturbed update $\hat{y} = D(x_t, y_{t-1})$.  The adversary can then choose any $w_t$ satisfying $\norm{w_t} \leq \frac{\rho - \alpha \rho }{ 1 + \rho }\cdot(1 - \norm{\hat{y}_t})$; suppose each $w_t$ is given by
\begin{align*}
    w_t =&\; \hat{y}_t \cdot \frac{\frac{\rho - \alpha \rho }{ 1 + \rho }\cdot(1 - \norm{\hat{y}_t})}{\norm{\hat{y}_t}}
\end{align*}
if $\hat{y}_t$ is non-zero, and an arbitrary vector $w_t$ with $\norm{w_t} = \frac{\rho - \alpha \rho }{ 1 + \rho }$ if $\hat{y}_t = \mathbf{0}$. 
This satisfies the disturbance norm bound, and further yields $y_t = \hat{y}_t + w_t$, where for non-zero $\hat{y}$ we have
\begin{align*}
    y_t =&\; \hat{y_t} \cdot \parens{1 + \frac{\frac{\rho - \alpha \rho }{ 1 + \rho }\cdot(1 - \norm{\hat{y}_t})}{\norm{\hat{y}_t}} }
\end{align*}
and thus for any $\hat{y}$,
\begin{align*}
    \norm{y_t} \geq&\; \norm{\hat{y}_t} + \frac{\rho - \alpha \rho }{ 1 + \rho }\cdot(1 - \norm{\hat{y}_t}) \\
    \geq&\; \frac{\rho - \alpha \rho }{ 1 + \rho },
\end{align*}
yielding a loss $f_t(y_t) \geq L \cdot \frac{\rho - \alpha \rho }{ 1 + \rho }$ at a disturbance cost of $\norm{w_t} = \frac{\rho - \alpha \rho }{ 1 + \rho }(1 - \norm{\hat{y}_t})$. Assuming the adversary continues this strategy in each round until any disturbance budget $E = \sum_{t=1}^T \norm{w_t}$ is exhausted, this yields a regret for any algorithm of at least
\begin{align*}
    \Reg_{T}(\A) \geq&\; \min\parens{LE, \frac{\rho - \alpha \rho }{ 1 + \rho }TL},
\end{align*}
as $y^* = \mathbf{0}$ obtains total loss 0.
\end{proof}
The disturbance upper bound is indeed necessary for $\rho$-locally controllable dynamics. We show a sharp threshold effect at $\frac{\rho}{1 + \rho} \cdot \pi( D(x_t, y_{t-1}))$, wherein an adversary who is allowed to exceed this limit by any amount can force an algorithm to incur linear regret even with only a constant budget. Note that for any $\rho \in (0,1]$ and $\alpha < 0$, there is some $\beta \in [0,1)$ such that $\frac{\rho - \alpha \rho}{1 + \rho} \geq \frac{\rho}{1 + \beta\rho}$.
\begin{theorem} \label{thm:oenftrlap-lower-bound}
Suppose an adversary can choose  any state disturbances $w_t$ with $\norm{w_t} \leq \frac{\rho}{1 + \beta \rho} \cdot \pi\parens{ D(x_t, y_{t-1}) }$, for any $\rho \in (0,1]$ and any $\beta \in [0,1)$. Then, there is a $\rho$-locally controllable instance $(\X, \Y, D)$ with convex losses $f_t$ such that any algorithm $\A$ obtains regret $\Reg_{T}(\A) = \Theta(T)$ even if $\sum_{t=1}^T \norm{w_t} = O(1)$. 
\end{theorem}
\begin{proof}
    Consider any instance $(\X, \Y, D)$ where $\rho$-local controllability exactly characterizes the range of $D$, i.e.\ for any $y$ and $y'$, there is some $x$ such that $D(x, y) = y'$ if and only if $y' \in \B_{\rho \cdot \pi(y)}(x, y)$.

Let $d_t = \pi(y_t)$ for each round. Beginning at any round $t$, suppose the adversary observes an action $x_{t}$ which yields an update 
$\hat{y}_{t} = D(x_{t}, y_{t-1})$. 
Let $z_{t} = \argmin_{y \in \bd(\Y)} \norm{y - \hat{y}_t}$, and suppose the adversary chooses the disturbance:
\begin{align*}
    w_t =&\; \argmin_{w : \norm{w} \leq \frac{\rho}{1 + \beta \rho} \cdot \pi\parens{ \hat{y}_t }} \norm{\hat{y}_t + w_t - z_t}.
\end{align*}
This forces $y_t$ closer to the boundary at each round, regardless of the choice of $x_t$: 
\begin{align*}
    d_{t} =&\; \parens{1 - \frac{\rho}{1 + \beta \rho}} \cdot \pi(\hat{y}_t) \\
    \leq&\; \parens{1 + \rho   -  \frac{\rho}{1 + \beta \rho} - \frac{\rho^2}{1 + \beta \rho} } d_{t-1} \tag{$\pi(\hat{y}_t) \leq (1 + \rho) d_{t-1}$}\\
    \leq&\;  \frac{1 +  \beta\rho  + \beta\rho^2  - \rho^2}{1 + \beta \rho}  d_{t-1} \\
    \leq&\; \parens{1   - \frac{(1 - \beta) \rho^2}{1 + \beta \rho} } d_{t-1},
\end{align*}
where $\pi(\hat{y}_t) \leq (1 + \rho) d_{t-1}$ holds by our assumption on $D(x,y)$. Assuming the adversary applies a disturbance $w_t$ selected as above in each round $t \leq T$, we have that
\begin{align*}
    d_{t} \leq&\; \parens{1   - \frac{(1 - \beta) \rho^2}{1 + \beta \rho} }^{t} \cdot d_{0},
\end{align*}
where the magnitude of each disturbance is bounded by
\begin{align*}
   \norm{w_t} \leq&\; \frac{\rho + \rho^2}{1 + \beta \rho} d_{t-1} \\
   \leq&\; \frac{\rho + \rho^2}{1 + \beta \rho} \parens{1   - \frac{(1 - \beta) \rho^2}{1 + \beta \rho} }^{t-1} \cdot d_{0}, 
\end{align*}
where we take the initial state distance to the boundary $d_0 = \pi(y_0)$ to be a constant bounded away from zero. 
This yields that the sum of disturbance magnitudes $E = \sum_{t=1}^T \norm{w_t}$ is at most:
\begin{align*}
    \sum_{t=1}^T \norm{w_t} \leq&\; d_0 \frac{\rho + \rho^2}{1 + \beta \rho} \cdot  \sum_{t=1}^T \parens{1   - \frac{(1 - \beta) \rho^2}{1 + \beta \rho} }^{t-1} \\
    \leq&\; d_0 \cdot \frac{\rho + \rho^2}{(1 - \beta) \rho^2} \\
    =&\; O(1).
\end{align*}
Now suppose that the loss at each round is given by $f_t(y_t) = \norm{y_t - y_0}$. Then, our regret with respect to $y_0$ is at least:
\begin{align*}
    \sum_{t=1}^T f_t(y_t) - f_t(y_0) \leq&\; \sum_{t=1}^T d_0 - d_t \\
    \leq&\; d_0 \parens{T - \sum_{t=1}^T \frac{(1 - \beta)\rho^2}{1 + \beta \rho}} \\
    \leq&\; d_0\parens{T - \frac{1 - \frac{(1 - \beta) \rho^2}{1 + \beta \rho} }{\frac{(1 - \beta) \rho^2}{1 + \beta \rho} }} \\
    \leq&\; d_0\parens{T - \frac{1 + \beta \rho}{(1 - \beta) \rho^2} } \\
    =&\; \Theta(T).
\end{align*}
\end{proof}
Together, the previous three theorems yield Theorem \ref{thm:oenftrlap-body}.

\subsection{$\oenftrluap$ and Proofs for Theorem \ref{thm:oenftrluap-body}}

We can remove the bound on the maximum disturbance for strongly locally controllable instances, as the feasible update sets do not vanish at the boundary of $\Y$.
Recall that an instance $(\X, \Y, D)$ satisfies strong $\rho$-local controllability for $\rho > 0$ if, for  any $y \in \Y$ and $y^* \in \B_{\rho}(y)\cap \Y$, there is some $x$ such that $D(x, y) = y^*$. We assume without loss of generality that $\rho \leq 2R$, where $R$ is the radius of $\Y$.

Intuitively, our algorithm tracks the target state which would be chosen by $\ftrl$ in the absence of all disturbances (by recording the loss counterfactual loss rather than the one truly experienced), and always seeks to minimize distance to that state.
\begin{algorithm}
\caption{$\oenftrl$ with Unbounded Disturbances ($\oenftrluap$).}\label{alg:nonlinear-ftrl-uap}
\begin{algorithmic}
\STATE Initialize $\ftrl$ for $T$ rounds over $\Y$ with step size $\eta = \sqrt{\frac{G\gamma}{T L^2}}$.
\FOR{$t = 1$ to $T$}
\STATE Let $\hat{y}_t$ be the target state chosen by $\ftrl$.
\STATE Use $\texttt{Oracle}(y_{t-1}, \hat{y}_t)$ to compute $x_t = \argmin_{x \in \X} \norm{ D(x, y_{t-1}) - \hat{y}_{t}}^2$.
\STATE Play action $x_t$.
\STATE Observe disturbed state $y_t =  D(x_t, y_{t-1}) + w_t$ and loss $f_t(y_t)$.
\STATE Update $\ftrl$ with state $\hat{y}_t$ and loss $f_t(\hat{y}_t)$.
\ENDFOR
\end{algorithmic}
\end{algorithm}

\begin{theorem}\label{thm:oenftrluap-regret}
For a strongly $\rho$-locally controllable instance $(\X, \Y, D)$ with convex 
losses $f_t : \Y \rightarrow \R$ and adversarial disturbances $w_t$ where  $\sum_{t=1}^T \norm{w_t} \leq E$, the regret of $\oenftrluap$ is bounded by
\begin{align*}
    \Reg_{T}(\oenftrluap) \leq&\;  O\parens{\sqrt{T} + E \cdot \rho^{-1} }
\end{align*}
with respect to the reward of any state, with $T$ queries made to an oracle for non-convex optimization.
\end{theorem}
\begin{proof}
We begin by bounding the total state error $\sum_{t=1}^t \norm{y_t - \hat{y}_t}$ across rounds. First, note that for any fixed $\rho > 0$, and any desired  $\alpha \in (0,1)$, we have that $\eta \frac{L}{\gamma} \leq \rho \alpha$ for sufficiently large $T$, as $\eta \frac{L}{\gamma} = \sqrt{\frac{G}{T \gamma}}$; we assume this holds for any given choice of $\alpha$, and so we have that $\norm{\hat{y}_{t+1} - \hat{y}_{t}} \leq \rho \alpha$ by Proposition \ref{prop:ftrl-step}.
For a total disturbance budget $E$, we separately consider disturbances $w_t$ depending on whether or not the accumulated disturbance error up to $w_t$ is driven to 0 in the next round. Define $W_{+}$ and $W_{-}$ as:
\begin{align*}
    W_{+} =&\; \{w_{t} : D(x_{t+1}, y_{t}) \neq \hat{y}_{t+1} \} 
\end{align*} 
and
\begin{align*}
    W_{-} =&\; \{w_{t} : D(x_{t+1}, y_{t}) = \hat{y}_{t+1} \}
\end{align*}
with $E_{+} = \sum_{w_t \in W_{+}} \norm{w_t}$  and $E_{-} = \sum_{w_t \in W_{-}} \norm{w_t}$.
First, observe that at each round $t$ corresponding to $w_t \in W_{-}$, given that $\norm{\hat{y}_{t+1} - y_t} \leq \rho$ we have that $\norm{w_t} = \norm{y_t - \hat{y}_t} \leq (1 + \alpha)\rho$, as $\norm{\hat{y}_{t+1} - \hat{y}_t} \leq \alpha\rho$. As such, we have that
\begin{align*}
    \sum_{t : w_t \in W_{-}} f_t(y_t) - f_t(\hat{y}_t) \leq&\;  \sum_{t : w_t \in W_{-}} L\norm{y_t - \hat{y}_t} \\
    \leq&\; (1 + \alpha)LE_{-}.
\end{align*}
Next, consider any $w_t \in W_{+}$. As our instance is strongly $\rho$-locally controllable, we must have that $\norm{\hat{y}_{t+1} - y_t} > \rho$, as otherwise there would some feasible action $x_{t+1}$ which would be selected that would yield $w_t \in W_{-}$. Since  $\norm{\hat{y}_{t+1} - \hat{y}_t} \leq \alpha\rho$, it then must be the case that $\norm{w_t} = \norm{y_t - \hat{y}_t} > (1 - \alpha)\rho$, and so we can bound the number of disturbances in $W_{+}$ as:
\begin{align*}
    \abs{W_{+}} \leq&\; \frac{E_{+}}{(1 - \alpha)\rho}.
\end{align*}
Assuming a maximal distance $\norm{\hat{y}_{t} - y_t} = 2R$ for each round $t$ corresponding to some $w_{t} \in W_{+}$, this yields
\begin{align*}
    \sum_{t : w_t \in W_{+}} f_t(y_t) - f_t(\hat{y}_t) \leq&\;  \sum_{t : w_t \in W_{+}} L\norm{y_t - \hat{y}_t} \\
    \leq&\; \frac{2LRE_{+}}{(1 - \alpha)\rho}
\end{align*}
We can assume $\alpha$ is small enough to yield $\frac{2R}{\rho} \geq (1 + \alpha) \cdot (1 - \alpha)$, and so we have
\begin{align*}
    \sum_{t =1}^T f_t(y_t) - f_t(\hat{y}_t) \leq&\; \frac{2LRE}{(1 - \alpha)\rho}.
\end{align*}
The regret bound for $\ftrl$ holds over the states $\hat{y}_t$, and so we can bound the total regret of $\oenftrlap$  with respect to any $y^* \in \Y$ as:
\begin{align*}
    \sum_{t=1}^T f_t(y_t) - f_t(y^*) \leq&\; \sum_{t=1}^T f_t(\hat{y}_t) - f_t(y^*) + \sum_{t =1}^T f_t(y_t) - f_t(\hat{y}_t) \\
    \leq&\; \eta \frac{TL^2}{\gamma} + \frac{G}{\eta} + \frac{2LRE}{(1 - \alpha)\rho} \tag{Prop.\ \ref{prop:ftrl}} \\
     \leq&\; 2 \sqrt{\frac{ TGL^2}{\gamma}}+ \frac{2LRE}{(1 - \alpha)\rho}. 
\end{align*}
\end{proof}

\begin{theorem}[Regret Lower Bound for Unbounded Disturbances]\label{thm:oenftrluap-lower}
Suppose an adversary can choose any state disturbances $w_t$ with 
$\sum_{t=1}^T \norm{w_t} = E$.
For any $\rho \in (0,1]$, there is a strongly $\rho$-locally controllable instance $(\X, \Y, D)$ with convex losses $f_t$ such that any algorithm $\A$ obtains regret $\Reg_{T}(\A) =  \min(\frac{2LRE}{\rho}, 2TLR)$. 
\end{theorem}
\begin{proof}
Let $\Y = [-R, R]$ for any $R > 0$ and let $f_t(y_t) = -Ly_t + LR$ for each $y$. Suppose strong $\rho$-local controllability exactly characterizes the range of $D$, i.e.\ for any $y, y' \in \Y$ there is some $x$ such that $D(x,y) = y'$ if and only if $\abs{y-y'} \leq \rho$. Consider an adversary who chooses disturbances $w_t$ in each round such that $y_t = -R$ until their disturbance budget $E$ is exhausted. This requires a disturbance of magnitude at most $R + \rho$ for $w_1$, as we assume $y_0 = 0$, and at most $\rho$ in subsequent rounds, and thus the adversary can force any algorithm to remain at $y_t = -R$ for $({E - R}){\rho^{-1}}$ rounds.  

As such, any algorithm must incur loss of at least ${2LR(E - R)}\rho^{-1}$ across these rounds, and further must incur average loss $LR$ over the subsequent ${2R} \rho^{-1}$ rounds (if $T$ is not yet reached), for an additional loss of ${2LR^2}{\rho^{-1}}$, as they can only decrease per-round loss by $L\rho$ given the restriction on the range of $D$. As the optimal state $y^* = R$ obtains loss 0, the total regret is at least:
\begin{align*}
    \sum_{t=1}^T f_t(y_t) - f_t(y^*) \geq&\; \min\parens{ \frac{2LRE}{\rho} , 2TLR }.
\end{align*}
\end{proof}
Together, the previous two theorems yield Theorem \ref{thm:oenftrluap-body}. Note that for both algorithms it remains computationally efficient to optimize over action-linear dynamics, as the constraint that $D(x, y_{t-1}) \in \Y$ can be encoded as a convex contraint over $\X$.

\section{Unknown Dynamics: Analysis for $\probingoco$}\label{sec:appendix-unknown}

\begin{algorithm}
\caption{Probing Online Convex Optimization ($\probingoco$).}\label{alg:probing-oco}
\begin{algorithmic}
\STATE Let $n = \dim(\X)$, let $y_0 = \mathbf{0}$, and let $x_1 \in \X$ such that $\norm{ D(x_1, y_0) - y_0 } \leq \epsilon = o(\sqrt{T})$  
\STATE Initialize $\oenftrlap$ to run over $\Y$ for $T/(2n+1)$ rounds
\STATE Run $\textsc{Estimate}$ for $2n +1 $ rounds:
\STATE Play $x_1$
\FOR{$i=1$ to $n$}
\STATE Play $x_1 + \epsilon \cdot e_i $
\STATE Play $x_1 - \epsilon \cdot e_i $
\ENDFOR
\STATE Solve for estimates $(\hat{A}_y, \hat{b}_y)$ which are consistent with with the previous $2n + 1$ 
 observed state updates, up to error $O(\epsilon)$
\FOR{$t = 2n + 1$ to $T$}
\STATE Let $t^* = t$ 
\STATE Using $(\hat{A}_y, \hat{b}_y)$, target $y = y_{t^*}$
\STATE Let $y^*$ be the next point chosen by $\oenftrlap$
\FOR{$i=1$ to $n$} 
\STATE Using $(\hat{A}_y, \hat{b}_y)$, target $y = y_{t^*} + \frac{2i - 1}{2n}(y^* - y_{t^*}) + \epsilon \cdot e_i$
\STATE Using $(\hat{A}_y, \hat{b}_y)$, target $y = y_{t^*} + \frac{2i}{2n}(y^* - y_{t^*}) - \epsilon \cdot e_i$
\ENDFOR
\STATE Update estimates $(\hat{A}_y, \hat{b}_y)$, solving for values which are consistent with the previous $2n + 1$ observed state updates, up to error $O(\epsilon)$
\ENDFOR
\end{algorithmic}
\end{algorithm}

\begin{proof}\textbf{of Theorem \ref{thm:unknown-dynamics}}
Assume the following hold for $D(x, y)$ at each $y$:
\begin{itemize}
    \item $D(x, y) = A_y \cdot x + b_y + y + q_y(x)$, for a function $q_y : \X \rightarrow \R^{n}$;
    \item $A_y$ has a largest absolute eigenvalue bounded by an absolute constant, smallest absolute eigenvalue bounded away from 0, and is $L_{\alpha}$-Lipschitz in the matrix $\ell_2$ norm;
    \item $b_y$ has a norm bounded by an absolute constant, and is $L_{\beta}$-Lipschitz;
    \item $\norm{q_y(x)} \leq \epsilon$ for any $x$ such that $\norm{ A_y \cdot x + b_y - y } = O(\sqrt{T})$.
\end{itemize}
In the neighborhood of any $y^*$, observe that playing $x = A^{-1}_y(y^* - y  - b_y)$ yields an update to $y^* + w_{\epsilon}$, where the error term $w_{\epsilon}$ has magnitude bounded  linearly in terms of the neighborhood size as well as polynomial in the relevant constants. We assume sufficiently small values of $\epsilon$, $L_{\alpha}$, and $L_{\beta}$ (whose relative bounds may trade off with each other, and in general will be inverse-polynomial in problem parameters other than $T$) to bound the error of this process in accordance with the requirements of Theorem \ref{thm:oenftrlap-body}, as well as to ensure that estimation error for $(\hat{A}_y, \hat{b}_y)$ is uniformly bounded for all $t \leq T$. 
Given $\epsilon = o(\sqrt{T})$, this yields estimation error terms $w_t \leq C \sqrt{T}$ in each round, for small enough $C$ to obtain the obtain the desired regret bound.
\end{proof}

\section{Bandit Feedback: Analysis for $\nestedbco$} \label{sec:appendix-bandit}

We first state the $\fkm$ algorithm and its bounds for regret and per-round step size.

\begin{algorithm}
    \caption{$\fkm$ \citep{FKM04}}
    \begin{algorithmic}
        \STATE Input: decision set $\K$ containing $\mathbf{0}$, set $v_1 = \mathbf{0}$,  parameters $\eta, \tilde{\delta}$.
        \STATE Let $v_1 \in \intr(\K)$ such that $\nabla \mathcal{R}(v_1) = 0$, 
        \FOR{$t = 1$ to $T$}
        \STATE Draw $u_t \in \mathbb{S}$ uniformly, set $y_t = v_t + \tilde{\delta} u_t$ 
        \STATE Play $y_t$, observe loss $f_t(y_t)$, set $g_t = \frac{n}{\tilde{\delta}} f_t(y_t) u_t$
        \STATE Update $v_{t+1} = {\Pi}_{\K_{\tilde{\delta}}} \brackets{v_t - \eta g_t } $, where  $\K_{\tilde{\delta}} = \{(1 - \tilde{\delta})v : v\in \K \}$
        \ENDFOR
    \end{algorithmic}
\end{algorithm}

\begin{proposition}[\citet{FKM04}] \label{prop:fkm}
For $L$-Lipschitz convex losses and a domain $\K$ with diameter $2R$ which contains a ball of radius $r$ around the origin, $\fkm$ obtains expected regret 
\begin{align*}
    \Reg_T(\fkm) \leq&\; \eta \frac{n^2}{\tilde{\delta}^2}T + \frac{4R^2}{\eta r^2 } + \frac{8\tilde{\delta} RLT}{r},
\end{align*}
with each point $y_t$ contained in $\K$. Further, each pair of consecutive points $y_{t}$, $y_{t+1}$ chosen by $\fkm$ satisfies $\norm{y_{t+1} - y_t} \leq 2\tilde{\delta} + \frac{\eta n L}{\tilde{\delta}}$. 
\end{proposition}

The $\nestedbco$ algorithm is essentially equivalent to $\oenftrl$, replacing $\ftrl$ with $\fkm$ and recalibrating parameters.

\begin{algorithm}
\caption{Nested Bandit Convex Optimization ($\nestedbco$).}\label{alg:nonlinear-ftrl}
\begin{algorithmic}
\STATE Let $\tilde{\delta} = \frac{1}{T^{1/4}} = r\delta\rho / 4$, let $\eta = \frac{R}{2nr LT^{3/4}}$
\STATE Let $\widetilde{\Y} = \{ y : \frac{1}{1-\delta}y \in \Y\}$ 
\STATE Initialize $\fkm$ to run for $T$ rounds over $\widetilde{\Y}$ with parameters $\eta, \tilde{\delta}$
\FOR{$t = 1$ to $T$}
\STATE Let $y^*$ be the point chosen by $\fkm$
\STATE Use $\texttt{Oracle}(y_{t-1}, y^*)$ to compute $x_t = \argmin_{x} \norm{ D_t(x, y_{t-1}) -  y^*}^2$ 
\STATE Play action $x_t$
\STATE Observe $y_t$ and loss $f_t(y_t)$, update $\scrible$
\ENDFOR
\end{algorithmic}
\end{algorithm}

\begin{proof}\textbf{of Theorem \ref{thm:bandit}.} 
Following the proof of Theorem \ref{thm:oenftrl}, to apply the bound of $\fkm$ to our setting (along with excess regret at most $\delta LR$ per round from contracting $\Y$ to $\widetilde{\Y}$), the key step is to show that each point selected by $\fkm$ is feasible under weakly locally controllable dynamics over $\widetilde{\Y}$, i.e.\ $\norm{y_{t+1} - y_t} \leq r\delta \rho$. Let $\tilde{\delta} = \frac{1}{T^{1/4}} = r\delta\rho / 4$, and let $\eta = \frac{R}{2nr LT^{3/4}}$. Assume for simplicity that $r \leq 1$ and $T^{1/4} \geq \frac{R}{r}$. When instantiating $\fkm$ over $\widetilde{\Y}$ with parameters $\eta$ and $\tilde{\delta}$, by Proposition \ref{prop:fkm} we then have
\begin{align*}
    \norm{y_{t+1} - y_t} \leq&\; 2\tilde{\delta} + \frac{\eta n L}{\tilde{\delta}} \\
    \leq&\; r\delta\rho / 2 + \parens{ \frac{ R}{ 2 nr LT^{3/4}} } \frac{n L }{\tilde{\delta}} \\
    \leq&\; r\delta\rho / 2 +   \tilde{\delta} /2 \\
    \leq&\; r \delta \rho,
\end{align*}
and so each selected point is feasible. This allows us to bound our regret by
\begin{align*}
    \Reg_T(\nestedbco) =&\; \Reg_T(\fkm) + \delta LRT \\
    =&\; \eta \frac{n^2}{\tilde{\delta}^2}T + \frac{4R^2}{\eta r^2 } + \frac{8\tilde{\delta} LRT}{r} + { \delta LRT} \\
    =&\; \eta \frac{16 n^2}{r^2 {\delta}^2 \rho^2}T + \frac{4R^2}{\eta r^2 } + 2\delta\rho LRT + { \delta LRT} \tag{$\tilde{\delta} = r\delta \rho/4$} \\
    \leq&\; 16 \eta n^2 T^{3/2}  + \frac{4R^2}{\eta r^2 } + \frac{12 LRT^{3/4}}{r\rho} \tag{$\delta = \frac{4}{r\rho T^{1/4}}, r \leq 1$} \\
    \leq&\; \frac{16 nL RT^{3/4}}{r}  + \frac{12 LRT^{3/4}}{r\rho} \tag{$\eta = \frac{R}{2nr LT^{3/4}}$} \\
    =&\;  O \parens{nRLT^{3/4} (r\rho)^{-1}}.
\end{align*}
\end{proof}

\section{Background and Proofs for Section \ref{subsec:app-pp}: Performative Prediction}
\label{sec:appendix-pp}

\subsection{Background}
Introduced by \cite{DBLP:journals/corr/abs-2002-06673}, the Performative Prediction problem captures settings in which the data distribution for which a classifier is deployed may shift as a function of the classifier itself, notably including strategic classification \cite{DBLP:journals/corr/HardtMPW15} as well as problems related to reinforcement learning and causal inference.
While a number of extensions of strategic classification to online settings have been considered \cite{SCrevealed,NEURIPS2021_812214fb,ahmadi2023fundamental}, the bulk of the literature on performative prediction considers settings with a fixed loss function and distribution ``update map''  \cite{DBLP:journals/corr/abs-2002-06673,echochamber,PPregret,NEURIPS2020_33e75ff0,piliouras2022multiagent,brown2022performative}, where the update map may sometimes depend on the current distribution (as in the Stateful Performative Prediction setting of \cite{brown2022performative}). For the {\it location-scale} family of update maps introduced by \cite{echochamber} (and additionally explored by \cite{PPregret} from a regret minimization perspective), which yields a convex ``performative risk'' objective function, a formulation of Online Performative Prediction is given by \cite{kumar2022online} as an application of online convex optimization with unbounded memory, in which the classification loss function may change over time and the distribution updates may occur gradually. 

Here, we generalize the problem formulation of \cite{kumar2022online} to also accommodate notions of statefulness similar to that in \cite{brown2022performative}. In particular, the instances we consider will resemble location-scale maps when restricting attention only the performatively stable classifiers for each distribution, yet the update effect of a non-stable classifier may be distribution-dependent and nonlinear, provided that the update map satisfies local controllability (viewing classifiers as actions and distributions as states) and mild regularity properties (e.g.\ invertibility and Lipschitz conditions).

\subsection{Model}
In the setting of Online Performative Prediction we consider, as formulated by \cite{kumar2022online}, in each round $t \in [T]$ we deploy some classifier $x_t$, and observe samples from some distribution $p_t$, which may change dynamically as a function of the history of interactions.
Here, we take $\X \subseteq \R^n$ as our space of classifiers, e.g.\ representing weight vectors for regression, which we assume is bounded and convex.
The initial data distribution is given by some distribution $p_0$ over $\R^n$. 
In each round, upon deploying a classifier $x_t$, the distribution is updated according to
\begin{align*}
    p_{t} =&\; (1 - \theta)p_{t-1} + \theta \D(x_t, y_{t-1}),
\end{align*}
for $\theta \in (0,1]$, where $\D(x_t, y_{t-1})$ is the distribution {\it update map} taking as input our classifier $x_t$ and some representation of the {\it state} $y \in \Y$, where we assume $\Y \subseteq \R^n$ is convex, contains $\B_{r}(\mathbf{0})$, is bounded with radius $R$, and that $y_0 = 0$. We make the following assumptions on $\D$.
\begin{assumption}\label{assum:PP-map}
We assume the distribution update map $\D(x, y)$ operates as follows:
\begin{itemize}
    \item $\D(x, y) = A(x,y) + \xi$, with $A : \X \times \Y \rightarrow \Y$,
    \item $\xi$ is a random variable in $\R^n$ with mean $\mu$ and covariance $\Sigma$,
    \item $A(x,y)$ satisfies $\rho$-local controllability and has an inverse action mapping $X(y,y^*)$ where
    \begin{align*}
        A(X(y, y^*), y) = y^*,
    \end{align*}
    defined over feasible pairs, which is $L_y$-Lipschitz in $y$ (when feasibility of $y^*$ holds), and
    \item There is a linear invertible function $s : \X \rightarrow \Y$ such that $A(x, y) = s(x)$ if $y = s(x)$,
     where $s^{-1} : \Y \rightarrow \X$ is $S$-Lipschitz.
\end{itemize}
Further, $A(x,y)$ is known and $\xi$ can be sampled freely. 
\end{assumption}
The inverse action mapping assumption simply enforces that classifiers need not change drastically to have the same update effect under small changes to the state.
The final assumption imposes a linear structure over {\it performatively stable} classifiers (i.e.\ classifiers for which the resulting distribution will remain fixed under $\D$, as formulated by \cite{DBLP:journals/corr/abs-2002-06673}), but we note that the distribution may update in an arbitrarily nonlinear fashion (subject to the other conditions) when $x_t$ is not a performatively stable classifier for the distribution induced by the previous state $y_{t-1}$. The ability to accommodate a state component is reminiscent of prior work involving notions of statefulness in performative prediction such as \cite{brown2022performative}. Our setting generalizes that of \cite{kumar2022online}, in which the map $A$ is taken to be a fixed matrix. For any nonsingular matrix $A$ there is immediately a linear map $s(x) = A^{-1}x$, and local controllability can be defined in terms of the largest and smallest absolute eigenvalues of $A$ (as a special case of our Example \ref{ex:1} with a fixed matrix). We view the nonsingularity assumption (and invertibility in the more general case) as fairly mild, as it amounts to assuming that the distribution map can depend on all parameters of classifier without any necessary (linear) dependency structure imposed, and that no two classifiers are equivalent only to the population but not the optimizer (as otherwise one could simply reduce dimensionality of $\X$). However, even in the case where $A$ is singular, we note that this issue is resolvable augmenting the state representation $y_t$ to incorporate the choice of free classifier parameters which affect loss but not distribution updates (e.g.\ by adding a vector $w_t$ to $y_t$ which is orthogonal to the range of $A$ and linear in $x_t$). We assume invertibility here for simplicity, and we take $\Y$ to be simply be given by the range of $s$ over $\X$.
At each round $t$, some scoring function $f_t(x, z)$ is chosen adversarially, and our loss is then given by
\begin{align*}
    \tilde{f}_t(x_t, p_t) =&\; \E_{z\sim p_t}[f_t(x_t, z)].
\end{align*}
We assume each $f_t$ is convex and $L_z$-Lipschitz in both $x$ and $z$, and that 
$p_0 = y_0 + \xi$. 
We measure our regret with respect to the best performatively stable classifier, i.e.\ the loss of any classifier as if were held constant indefinitely as the distribution updates. 
We define our regret as follows:
\begin{align*}
    \Reg_T(\A) =&\; \max_{x^*} \sum_{t=1}^T \tilde{f}_t(x_t, p_t) - \tilde{f}_t(x^*, \D(x^*, s(x^*)))
\end{align*}
Here, the role of $s(x^*)$ captures the convergence of the distribution to a stable point, resulting from taking the limit of the distribution update rule as $t$ grows large.

As in many of the applications we consider, here our loss is determined both by our action (the classifier) and the state (in terms of the distribution). Our approach for casting Online Performative Prediction as an instance of online nonlinear control in our framework will be to define appropriate surrogate convex losses which depend only on the state, over which we run $\oenftrl$. Here, these will correspond to losses only over the updated distribution component $\D(x_t, y_{t-1})$, which we show closely track our true incurred loss.

\subsection{Analysis}
For each round $t$, define the surrogate loss $f^*_t(y)$ as:
\begin{align*}
    f^*_t(y) = \E_{z \sim y_t + \xi} \brackets{ f_t(s^{-1}(y), z)}.
\end{align*}

\begin{lemma}
Each $f_t^*(y)$ is convex and $(1 + S)L_z$-Lipschitz in $y$.
\end{lemma}
\begin{proof}
Consider any individual sample $v \sim \xi$. We can then view $g(y) = (s^{-1}(y), y+v)$ as a vector-valued function which is $(1 + S^*)$-Lipschitz. The function $f_t(g(y))$ is a $L_z$-Lipschitz and convex function of this linear function of $y$, and thus $f_t(s^{-1}(y), y + v)$ is convex and $(1+S^*)L_z$-Lipschitz in $y$. The function $f^*_t(y)$ is an average of such functions, taken over the expectation of $\xi$, and thus is convex and $(1+S^*)L_z$-Lipschitz in $y$ as well.
\end{proof}
Observe that $f^*_t(y) = \tilde{f}_t(s^{-1}(y), \D(s^{-1}(y), y)$.  We will run $\oenftrl$ for these losses over the $\rho$-locally controllable instance $(\X, \Y, A)$, where we can track the current state $y_t = A(x_t, y_{t-1})$ at each step as a function of our past actions given knowledge of $A$, and can compute gradients of $f^*_t(y_t)$ to arbitrary desired precision by sampling from $\xi$. This will yield the regret bound from Theorem \ref{thm:oenftrl} with respect to the surrogate losses, and the key challenge will be to analyze our error between the true and surrogate losses.
\begin{lemma}
For any round $t$ we have that
\begin{align*}
    \tilde{f}_t(x_t, p_t) - f_t^*(y_t) \leq&\; (1 - \theta)^{h} M +  \frac{ \eta L_z (1 + S)  }{\gamma } \cdot \parens{ L_y + \frac{1 - \theta}{\theta} }
\end{align*}
\end{lemma}
\begin{proof}
For any $h < t$, the loss of $x_t$ over the distribution $y_{t-h} + \xi = \D(x_{t-h}, y_{t - h - 1})$ can be expressed as
\begin{align*}
    \hat{f}_{t}(x_t, y_{t-h}) =&\; \E_{z \sim \xi + y_{t-h}}\brackets{ f_t(x_t, z) }, 
\end{align*}
which is convex and $L_z$-Lipschitz in both parameters when taking the expectation over $\xi$. For round $t$ in isolation,  using the inverse action mapping bound and the bound on $\norm{y_t - y_{t-1}}$ from Proposition \ref{prop:ftrl-step} we have that
\begin{align*}
    \hat{f}_{t}(x_t, y_{t}) - f^*_t(y_t) =&\; \hat{f}_{t}(x_t, y_{t}) - \hat{f}_t(s^{-1}(y_t) , y_t) \\
    =&\; \hat{f}_{t}(X(y_{t-1}, y_t), y_{t}) - \hat{f}_t(X(y_t, y_t) , y_t) \\
    \leq&\;  \frac{\eta L_y L_z}{\gamma},
\end{align*}
and further for previous states that 
\begin{align*}
    \hat{f}_{t}(x_t, y_{t-h}) - f^*_t(y_t) =&\; (L_y + h) \frac{\eta L_z(1 + S) }{\gamma}.
\end{align*}
We can decompose the distribution $p_t$ into updates from past rounds as
\begin{align*}
    p_t =&\; (1 - \theta)^{t}p_0 + \sum_{h=0}^{t-1} \theta(1 - \theta)^{h} \D(x_{t-h}, y_{t - h-1})
\end{align*}
which then yields a loss discrepancy of at most
\begin{align*}
    \tilde{f}_t(x_t, p_t) - f_t^*(y_t) \leq&\;   (1 - \theta)^{t} f_t(x_t, p_0) + \frac{\eta L_z(1 + S) }{\gamma} \parens{ \sum_{h=0}^{t-1} \theta(1 - \theta)^{h} (L_y + h) }  \\
    \leq&\;  \frac{ \eta L_z (1 + S)  }{\gamma } \cdot \parens{ L_y + \frac{1 - \theta}{\theta} + (1 - \theta)^{t}  }
\end{align*}
between the true and surrogate loss for round $t$.
\end{proof}
We can now bound the cumulative regret of $\oenftrl$ for the problem.
\begin{theorem}\label{thm:perf-appendix-main}
    For any $\theta > 0$, when Assumption \ref{assum:PP-map} holds for the distribution update rule, Online Performative Prediction can be cast as a $\rho$-locally controllable instance of online control with nonlinear dynamics, for which $\oenftrl$ obtains regret
\begin{align*}
    \Reg_T(\oenftrl) \leq&\; 2 \sqrt{\frac{(1 + L_y + \frac{R}{r\rho} + \frac{2 - \theta}{\theta}) TGL_z^2(1 + S)^2}{\gamma}} 
\end{align*}
with respect to the best performatively stable classifier classifier.
\end{theorem}
\begin{proof}
Combining the previous results with Theorem \ref{thm:oenftrl}, we have that for any $x^* \in \X$ our regret is at most
\begin{align*}
 \sum_{t=1}^T \tilde{f}_t(x_t, p_t) - \tilde{f}_t(\D(x^*, s(x^*))) \leq&\; \sum_{t=1}^T \hat{f}_t(y_t) - \tilde{f}_t(x^*, \D(x^*, s(x^*)))  + \sum_{t=1}^T \tilde{f}_t(x_t, p_t) - {f}^*_t(y_t) \\
 \leq&\;  \eta \parens{1 + L_y +  \frac{2 - \theta}{\theta}  +  \frac{R}{r\rho}}  \frac{TL_z(1 + S) }{\gamma} + \frac{G}{\eta}  \\
  =&\; 2 \sqrt{\frac{(1 + L_y + \frac{R}{r\rho} + \frac{2 - \theta}{\theta}) TGL_z^2(1 + S)^2}{\gamma}} 
\end{align*}
upon setting $\eta = \sqrt{ \frac{ G \gamma }{(1 + L_y + \frac{R}{r\rho} + \frac{2 - \theta}{\theta}) TL_z^2(1 + S)^2} }$. 
\end{proof}
Theorem \ref{thm:pp-reg-body} follows directly from Theorem \ref{thm:perf-appendix-main}. For Online Performative Prediction, in the full generality of the setting considered, the per-round optimization problem may not be convex, in which case we make use of the non-convex optimization oracle access for $\oenftrl$. However, in each of the following applications we show that the action selection step can indeed be implemented efficiently without imposing additional restrictions on the dynamics.

\section{Background and Proofs for Section \ref{subsec:app-recs}: Adaptive Recommendations}
\label{sec:appendix-recs}

\subsection{Background}
Motivated by problems involving preference dynamics and feedback loops in recommendation systems (see e.g.\cite{flaxman}), a number of recent works \cite{Hazla+19,GaitondeKT21,dean2022preference,jagadeesan2022supplyside,AB22,agarwal2023online} have explored models of repeated recommendation where given to an agent whose preferences or opinions evolve over time. Several of these models \cite{Hazla+19,dean2022preference,jagadeesan2022supplyside} consider population-level effects for settings where a single recommendation is given each round and consumers (or producers) update their behavior according to linear dynamics. Nonlinear preference dynamics with {\it menus} of recommendations for a single agent are considered in \cite{AB22,agarwal2023online}, where the aims to minimize regret for adversarial losses over the agent's choices. The Adaptive Recommendations formulation of \cite{AB22} somewhat resembles the ``Dueling Bandits'' setting of \cite{YUE20121538}, where $k > 1$ actions are chosen in each round, yet where preferences can now evolve dynamically as a function of the history rather than remaining fixed. Whereas \cite{AB22,agarwal2023online} study a bandit formulation of the problem with unknown preference dynamics, here we consider a full-feedback model with known dynamics, allowing for relaxed structural assumptions (on the agent's ``memory horizon'' and ``preference scoring functions'') at the cost of stronger informational assumptions, while maintaining the overall dynamics of the problem.

\subsection{Model}
Here, we are tasked with repeatedly recommending menus of content to an agent. Out of a universe of $n$ elements (e.g.\ video channels, clothing items), we show a subset of size $k$ (denoted $K_t$) to the agent in each round, for $T$ total rounds. The agent chooses one item $i \in K_t$ from the menu, according to a distribution in terms of their {\it preferences}, which are a function of their selection history. Conditioned on being shown a menu $K_t$, the agent's choice distribution 
has positive mass only on the $k$ items $i \in K_t$. The agent's representation of their selection history is given by their {\it memory vector} $v_t \in \Delta(n)$, and choices are determined by their {\it preference scoring functions} $s_i : \Delta(n) \rightarrow [\lambda, 1]$ for each $i$, which map the agent's memory vector to relative preference scores for each item. 
The menu we show to the agent may be chosen from some distribution $x_t \in  \Delta({n \choose k})$, and for each $K_t \in [{n \choose k}]$ the agent's menu-conditional distribution $p_t( \cdot ; K_t, v_{t-1}) \in \Delta(n)$ is proportional to the scores $s_i(v_t)$ for items in $K_t$, given as
\begin{align*}
    p_t(i; K_t, v_{t-1}) =&\; \frac{s_{i}(v_{t-1})}{\sum_{j \in K_t} s_j(v_{t-1})} 
\end{align*}
for each $i\in K_t$, with $p_t(j; K_t, v_{t-1}) = 0$ for $j \notin K_t$. The joint item choice distribution, considering both random selection of a menu $K_t$ according to $x_t$, and the agent's choice from $K_t$, is given by
\begin{align*}
    p_t(\cdot; x_t, v_{t-1} ) =&\; \sum_{K_t \in {n \choose k}} x_t(K_t) \cdot p_t(\cdot ; K_t, v_{t-1})
\end{align*}
which we may denote simply by the vector $p_t \in \Delta(n)$, or as a function $p_t(x_t)$. In contrast to prior work, here we consider a deterministic variant of the problem as an illustration of the flexibility of our framework for online nonlinear control. In particular, we assume that the agent's memory vector $v_t$ updates according to its expectation over $p_t$ as
\begin{align*}
    v_{t} = (1 - \theta_t)v_{t-1} + \theta_t p_t,
\end{align*}
where $\theta_t \in [\theta,1]$ is the per-round update speed, and we assume that the agent's scoring functions $s_i$ are known. We receive convex and $L$-Lipschitz losses $f_t(p_t)$ in each round in terms of the agent's choices, over which we aim to minimize regret with respect to some distribution set $\Y \subseteq \Delta(n)$. 

The prior work \citep{AB22,agarwal2023online} has considered two particular subsets of $\Delta(n)$ as regret benchmarks. We show that both can be cast as locally controllable instances of online control, and further, we make use of local controllability to give a general characterization of convex sets $\Y \subseteq \Delta(n)$ over which sublinear regret is attainable. We recall some key definitions and results from \citep{AB22,agarwal2023online}.
\begin{definition}[Instantaneously Realizable Distributions]
The set of instantaneously realizable distributions at a memory vector $v \in \Delta(n)$ is given by
\begin{align*}
    \IRD(v) =&\; \convhull\braces{ p(\cdot ; K, v) : K \in \brackets{ {n \choose k} }  }. 
\end{align*}
\end{definition}
Each such set  $\IRD(v_{t-1})$ corresponds to the feasible distributions $p_t$, given the agent's scoring functions and memory $v_{t-1}$.
It is shown by \cite{agarwal2023online} that each $\IRD$ sets can be directly characterized in terms of the ratios between target frequencies and scores.
\begin{proposition}[Menu Times for $\IRD$ \cite{agarwal2023online}]
Given a memory vector $v \in \Delta(n)$ and target distribution $p \in \Delta(n)$, let the menu time $\mu_i$ for item $i$ be given by
\begin{align*}
    \mu_i =&\; \frac{k \cdot \frac{p(i)}{s_i(v)}}{\sum_{j=1}^n \frac{p(j)}{s_j(v)}},
\end{align*}
where $\sum_{i=1}^n \mu_i = k$. Then, $p \in \IRD(v)$ if and only if $\mu_i \leq 1$ for each $i \in [n]$.
\end{proposition}
We recall the prior benchmark sets considered, and the corresponding assumptions which yield feasibility of regret minimization. We state informal analogues of the prior results as translated to our setting, which we then show formally below.
\begin{definition}[Everywhere Instantaneously Realizable Distributions]
The set of everywhere instantaneously realizable distributions is given by 
\begin{align*}
    \EIRD =&\; \bigcap_{v \in \Delta(n)} \IRD(v).
\end{align*}
\end{definition}
\begin{proposition}[Corollary of \cite{AB22}]\label{prop:rec-menu-times}
If $\lambda \geq \frac{k}{n} + \frac{k}{n(n-1)}$, then $\EIRD$ is non-empty, and there is a $o(T)$ regret algorithm with respect to any distribution $p \in \EIRD$.
\end{proposition}
Distributions $p_t \in \EIRD$ are always feasible regardless of $v_{t-1}$ by an appropriate choice of $x_t$, but $\EIRD$ may be quite small in relation to $\Delta(n)$. Under stronger assumptions for each $s_i$, a potentially much larger set becomes feasible as a regret benchmark.
\begin{definition}[$\phi$-Smoothed Simplex] The $\phi$-{smoothed simplex} $\Delta^{\phi}(n)$ for $\phi \in [0,1]$ is given by 
\begin{align*}
   \Delta^{\phi}(n) =&\; \{(1 - \phi)v + \phi \mathbf{u}_n : v \in \Delta(n) \}
\end{align*}
\end{definition}
\begin{definition}[Scale-Bounded Functions]
A scoring function $s_i : \Delta(n) \rightarrow [\frac{\lambda}{\sigma}, 1]$ is said to be $(\sigma, \lambda)$-{scale-bounded} for $\sigma > 1$ and $\lambda > 0$ if, for all $v\in \Delta(n)$, we have that
\begin{align*}
    \sigma^{-1} ((1 - \lambda)v_i + \lambda) \leq s_i(v) \leq \sigma ((1 - \lambda)v_i + \lambda).
\end{align*}
\end{definition}
For such functions, each score $s_i(v)$ cannot be too far from item $i$'s weight in memory, and it is shown that $\IRD(v)$ contains a ball around $v$ for each $v \in \Delta^{\phi}(n)$, for an appropriate choice of $\phi$.
\begin{proposition}[Corollary of \cite{agarwal2023online}] If each $s_i$ is $(\sigma, \lambda)$-{scale-bounded}, then there is a $o(T)$ regret algorithm with respect to any distribution $p \in \Delta^{\phi}(n)$, for $\phi = \Theta(k\lambda\sigma^2)$. 
\end{proposition}
We extend these results to general convex benchmark sets $\Y \subseteq \Delta(n)$, where we can characterize the feasibility of regret minimization via local controllability using the menu times $\mu_i$. When $\rho$-local controllability holds over a set $\Y$, we can minimize regret via $\oenftrl$ using surrogate losses $f_t^*(v_t)$, which closely track our true losses $f_t(p_t)$. 
\subsection{Analysis}
We make use of the menu time quantities $\mu_i$ for a memory vector $v$ and target distribution $p$ to translate our notion of local controllability to the Adaptive Recommendations setting. 
Let $\Y$ be any convex subset of $\Delta(n)$, let $\X = \Delta({n \choose k})$, where the dynamics $D_t(x_t, v_{t-1})$ are given by
\begin{align*}
    D_t(x_t, v_{t-1}) =&\; (1 - \theta_t)v_{t-1} + \theta_t p_t(x_t).
\end{align*}
Note that $D_t(x_t, v_{t-1})$ is action-linear in $x_t$, and thus we can solve for $x_t$ efficiently (in terms of $\dim(\X) = O(n^k)$); further, there is a construction given in \cite{agarwal2023online} for removing exponential dependence on $k$ when computing menu distributions.
We consider $\Y$ as an $(n-1)$-dimensional subset of $\R^n$, where we define the the ball $\B_{\rho}(v)$ of radius $\rho$ around a point $v \in \Y$ as:
\begin{align*}
    \B_{\rho}(v) =&\; \{ p \in \Delta(n) : \norm{p - v} \leq \rho \}.
\end{align*}
\begin{theorem} \label{thm:rec-stable} An instance of Adaptive Recommendations $(\X, \Y, D)$ satisfies $\rho \theta$-local controllability if, for any $v \in \Y$ and $p \in \B_{\rho \cdot \pi(v)}$, we have that
\begin{align*}
    \frac{(k-1) p(i)}{s_i(v)} \leq&\; \sum_{j\neq i}^n \frac{ p(j)}{s_j(v)}
\end{align*}
for every $i \in [n]$.
\end{theorem}
This follows immediately from Proposition \ref{prop:rec-menu-times} and the definition of local controllability, which can analogously extend to strong local controllability. We can use this formulation to unify the feasibility analysis for each of the previously considered sets.
\begin{lemma}\label{thm:ar-eird-local}
 For $\lambda \geq \frac{k-1}{n-1} + \epsilon$ and $\epsilon \geq 0$, the $\EIRD$ set contains a ball of radius $\rho = \Theta(\frac{\epsilon}{nk + \epsilon})$ around $\mathbf{u}_n$, and any instance $(\X, \EIRD, D)$ satisfies $\theta$-local controllability.
\end{lemma}
\begin{proof}
For any 
$v \in \Delta(n)$, 
$i \in [n]$, and $p \in \B_{\rho}(\mathbf{u}_n)$ we have $p(i) \leq \frac{1}{n} + \frac{\rho \sqrt{2}}{2}$ and $s_i(v) \geq  \frac{k-1}{n-1} + \epsilon$, yielding that
\begin{align*}
    \frac{(k-1) p(j)}{s_j(v)} \leq&\; \frac{1 +  \frac{\rho n \sqrt{2}}{2}}{\frac{n}{n-1} +  \frac{\epsilon n}{k-1}},
\end{align*}
and over all items $j \neq i$ (with $s_j(v) \leq 1$) we have
\begin{align*}
\sum_{j \neq i}^n \frac{ p(j)}{s_j(v)} \geq&\; 1 - \frac{1}{n} - \frac{\rho \sqrt{2}}{2}. 
\end{align*}
Observe that the bounds for each term are equalized at $\frac{n-1}{n}$ when $\rho = \epsilon = 0$, and so $\mathbf{u}_n \in \EIRD$ whenever $\lambda \geq \frac{k-1}{n-1}$. 
We can specify $\epsilon(\rho)$ in terms of $\rho$ to 
maintain equality, and thus inclusion of $p \in \EIRD$.
Taking $\epsilon(\rho)$ in terms of $\rho$ as
\begin{align*}
   \epsilon(\rho) =&\; \frac{ \rho n (k-1)}{\frac{2(n- 1)}{ \sqrt{2}n} - \rho } \\
   =&\; \frac{\frac{\rho n (k-1)\sqrt{2}}{2}}{\parens{1 - \frac{1}{n} - \frac{\rho\sqrt{2}}{2}}} \\
   =&\; (k-1) \parens{\frac{ \frac{1}{n} +  \frac{\rho  \sqrt{2}}{2}}{1 - \frac{1}{n} - \frac{\rho\sqrt{2}}{2}} - \frac{1}{n-1}}
\end{align*}
gives us that 
\begin{align*}
    \frac{1}{n-1} +  \frac{\epsilon(\rho) }{k-1} \geq&\; \frac{ \frac{1}{n} +  \frac{\rho  \sqrt{2}}{2}}{1 - \frac{1}{n} - \frac{\rho\sqrt{2}}{2}} 
\end{align*}
for $\rho \geq 0$, and so we maintain that $p \in \EIRD$. Inverting, we have 
\begin{align*}
    \rho(\epsilon) =&\; \frac{\epsilon \frac{2(n- 1)}{ \sqrt{2}n}}{n(k-1) + \epsilon}
\end{align*}
as the radius of a ball around $\mathbf{u}_n$ contained in $\EIRD$. To see that $\EIRD$ is $\theta$-locally controllable, consider any $v_{t-1}$ and $v^*$ in $\EIRD$ where $v^* \in \B_{\pi(v_{t-1})}(v_{t-1})$, and let $v_t = (1 - \theta_t)v_{t-1} + \theta_t v^*$. By playing an action distribution $x_t$ which induces $p_t(x_t) = v^*$, the memory vector is then updated to $v_t$. This is feasible for any $v_t \in \B_{\theta \cdot \pi(v_{t-1})}(v_{t-1})$, as each corresponds to some $v^* \in \B_{\pi(v_{t-1})}(v_{t-1})$.
\end{proof}
We remark that for the $\EIRD$ set, if losses are given over $p_t$ rather than $v_t$, one can define dynamics which directly consider the state to simply be the induced distribution $p_t$ in each round, which satisfies strong local controllability with any $p_t \in \EIRD$ feasible at each round; in general, we consider dynamics to view the memory vector as the state, as the feasible updates $p_t$ are a function of $v_t$.
Such is the case for the $\phi$-smoothed simplex, for which we can state an analogous local controllability result.

\begin{lemma}\label{thm:ar-ss-local}
If each $s_i$ is $(\sigma, \lambda)$-{scale-bounded}, then any instance $(\X, \Delta^{\phi}(n), D)$ over the $\phi$-smoothed simplex for $\phi = \Theta(k\lambda\sigma^2)$ satisfies $\Omega(\theta \lambda \phi)$-local controllability.
\end{lemma}
\begin{proof}
The following lemma from \cite{agarwal2023online} shows that a ball of distributions around any memory vector $v \in \Delta^{\phi}(n)$ is feasible under $\IRD(v)$.

\begin{lemma}[$\IRD$ for Scale-Bounded Preferences \cite{agarwal2023online}]
Let each $s_i$ be  $(\sigma, \lambda)$-scale-bounded 
with $\sigma \leq \sqrt{4(n-1)/k}$,
and let $v \in \Delta^{\phi}(n)$ be a vector in the $\phi$-smoothed simplex, for $\phi \geq \Theta{ k\lambda \sigma^2}$. Then, $p \in \IRD(v)$
for any vector $p \in  \B_{\lambda \phi}(v) \cap \Delta^{\phi}(n)$.
\end{lemma}
Let $d = \min(\lambda \phi, \pi(v_{t-1})) \leq \lambda \phi \pi(v_{t-1})$ for any $v_{t-1}$ in $\Delta^{\phi}(n)$. Any $v^* \in \B_{d}(v_{t-1})$ then is contained in $\IRD(v_{t-1})$, and so playing $x_t$ such that $p_t(x_t) = v^*$ yields an update to $v_t = (1 - \theta_t)v_{t-1} + \theta v^*$, which is feasible for any $v_t \in \B_{d\theta}(v_{t-1})$, and so $\Omega(\theta \lambda \phi)$-local controllability holds.
\end{proof}
For any such set $\Y$ which yields locally controllable dynamics for the instance $(\X, \Y, D)$, we can minimize regret over $\Y$ via $\oenftrl$, where we optimize with respect to the surrogate losses $f_t^*(v_t)$. Note that for our regret benchmark of the best per-round instantaneously distribution in $\Y$, any fixed vector $v^*$ which is instantaneously targeted across all rounds yields an item distribution $p_t = v^*$ in each round, and so $f^*_t(v^*) = f_t(p^*)$. We assume that $y_0$ is bounded inside $\Y$ (which typically will hold for $y_0 = \mathbf{u}_n$).

\begin{theorem}\label{thm:ar-app-alg}
    For any $\rho$-locally controllable instance $(\X, \Y, D)$ of Adaptive Recommendations with update speed $\theta > 0$, running $\oenftrl$ over the surrogate losses $f_t^*(v_t)$ yields regret 
\begin{align*}
    \Reg_T(\oenftrl) \leq&\; 2 \sqrt{\frac{(2 + \frac{R}{r\rho} + \frac{1}{\theta}) TGL^2}{\gamma}}
\end{align*}
with respect to the true losses $f_t(p_t)$ over $\Y$.
\end{theorem}
\begin{proof}
Beyond applying the regret bound for $\oenftrl$ from Theorem \ref{thm:oenftrl}, the key step here is to bound surrogate loss errors as:
\begin{align*}
    \sum_{t=1}^T f_t(p_t) - f_t(v^*) \leq&\; \sum_{t=1}^T f_t^*(v_t) - f_t(v^*) + \sum_{t=1}^T f_t(v_t) - f_t(p_t) \\
    \leq&\; \eta \parens{1 + \frac{R}{r\rho}}  \frac{TL^2}{\gamma} + \frac{G}{\eta}  + \sum_{t=1}^T f_t(v_t) - f_t\parens{\frac{v_t - (1 - \theta_t)v_{t-1}}{\theta_t}} \\
    \leq&\; \eta \parens{1 + \frac{R}{r\rho}}  \frac{TL^2}{\gamma} + \frac{G}{\eta}  + \sum_{t=1}^T f_t(v_t) - f_t\parens{v_{t-1} + \frac{v_t - v_{t-1}}{\theta_t}} \\
    \leq&\; \eta \parens{1 + \frac{R}{r\rho}}  \frac{TL^2}{\gamma} + \frac{G}{\eta}  + L\parens{1 + \frac{1}{\theta}} \sum_{t=1}^T \norm{v_t - v_{t-1}} \\
    \leq&\; \eta \parens{2 + \frac{R}{r\rho} + \frac{1}{\theta}}  \frac{TL^2}{\gamma} + \frac{G}{\eta}  \\
    =&\; 2 \sqrt{\frac{(2 + \frac{R}{r\rho} + \frac{1}{\theta}) TGL^2}{\gamma}}
\end{align*}
upon setting $\eta = \sqrt{ \frac{ G \gamma }{(2 + \frac{R}{r\rho} + \frac{1}{\theta}) TL^2} }$, which yields the theorem.
\end{proof}
Theorems \ref{thm:ar-eird} and \ref{thm:ar-ss} follow from Theorem \ref{thm:ar-app-alg}, as well as from Lemmas \ref{thm:ar-eird-local} and \ref{thm:ar-ss-local}, respectively.

\section{Background and Proofs for Section \ref{subsec:app-pricing}: Adaptive Pricing}
\label{sec:appendix-pricing}
\subsection{Background}

While there is a large literature on designing online mechanisms for pricing discrete goods via auctions \citep{adwords,DBLP:journals/corr/abs-2002-07331,DBLP:journals/corr/abs-2002-11137,DBLP:journals/corr/MorgensternR16,DBLP:journals/corr/abs-1901-06808,DBLP:journals/corr/abs-1711-09176}, there is comparatively little work related to online pricing problems for real-valued goods. Most work for such problems to date requires strong assumptions on valuation functions, often either assuming linearity \citep{jia2014online} or additivity \citep{agrawal2023dynamic}, or requiring approximability via discretization \citep{mussi2022dynamic}. Here, we introduce a novel formulation for an Adaptive Pricing problem which builds on the myopic-demand fixed-cost setting of \cite{DBLP:journals/corr/RothUW15}, which we extend to accommodate adversarial {\it consumption rates} for the agent (which affect demand, as a function of the agent's {\it reserves}) as well as adversarial production costs. As in \cite{DBLP:journals/corr/RothUW15}, our setting can accommodate general convex (increasing) production cost functions and concave (increasing) valuations for the agent, provided that valuations additionally are homogeneous; to our knowledge, this encompasses a much wider class of valuations and costs than considered by any prior work on no-regret dynamic pricing for real-valued goods.

\subsection{Model}
In each round $t$, an agent (the {\it consumer}) begins with goods reserves $y_{t-1} \in \R_{\geq 0}^n$ (with $y_0 = \mathbf{0}$), then consumes an adversarially chosen fraction $\theta_t \in [\theta, 1]$ of each good simultaneously (e.g. corresponding to their rate of manufacturing downstream items, using the goods as components), updating their reserves to $(1 - \theta_t)y_{t-1}$. We (the {\it producer}) show the consumer some vector $p_t \in \R_{+}^{n}$ of per-unit prices for each good, and the consumer purchases some bundle of goods $x_t$. The consumer's valuation function for reserves of goods is given by $v : \R_{+}^n \rightarrow \R_{+}$, and their selection of $x_t = x^*(p_t, \theta_{t}, y_{t-1})$ is given by
\begin{align*}
    x^*(p_t, \theta_{t}, y_{t-1}) =&\; \argmax_{x \in \R_{+}^n} v(x + (1 - \theta_t)y_{t-1}) - \langle p_t, x \rangle.
\end{align*}
We later discuss behavior of $x^*$ when the $\argmax$ is undefined; it will suffice for us to only consider price vectors for which it is defined.
This updates the consumer's reserves to $y_t = x_t + (1 - \theta_t)y_{t-1}$.
Upon seeing the consumer's purchased bundle $x_t$, we receive their payment $\langle p_t, x_t \rangle$ minus our production cost $c_t(x_t) : \R_{+}^n \rightarrow \R_{+}$, where $c_t$ is adversarially chosen. Our utility is then given by 
\begin{align*}
    f_t(p_t, x_t) =&\; \langle p_t, x_t \rangle - c_t(x_t).
\end{align*}
We make the following assumptions on production costs $c_t$ and the consumer's valuation $v$.
\begin{assumption}[Production Costs]\label{assum:pricing-prod}
We assume that for each $c_t$, the following hold over $\R_{+}^n$:
\begin{itemize}
    \item $c_t$ is non-negative, convex, and $L_c$-Lipschitz,
    \item $\lim_{\epsilon\rightarrow0} c_t(\epsilon \cdot \mathbf{1}) \leq C_0$ for some $C_0 \geq 0$, and
    \item $c_t(x) \geq \phi \norm{x} + C_0$  for some $\phi > 0$.
\end{itemize}
Further, each $c_t$ is revealed prior to setting prices $p_{t+1}$.
\end{assumption}
\begin{assumption}[Consumer Valuations]\label{assum:pricing-val}
We assume 
that the following hold over some set $\Y \subseteq \R^{n}_{+}$:
\begin{itemize}
    \item $v$ is non-negative, continuous, and  differentiable, 
    \item $v$ is strictly concave and increasing, 
    \item $v$ is $(\lambda, \beta)$-Hölder continuous for some $\lambda \geq 1$ and $\beta \in (0,1]$, i.e.\ $$\abs{v(y) - v(y')} \leq \lambda \norm{y - y'}^{\beta},$$
    and
    \item $v$ is homogeneous of degree $k$ for some $k \in (0,1)$, i.e.\ $v(b y) = b^k v(y)$ for any $b > 0$.
\end{itemize}
Further, $v$ is known to the producer.
\end{assumption}

Given the concavity assumption, we note that it is without loss of generality to assume that $k \in (0,1)$ for the homogeneity parameter.
There are several well-studied valuation families which satisfy these properties for an appropriate set $\Y$; see \cite{DBLP:journals/corr/RothUW15} for proofs of each example.

\begin{example}[Constant Elasticity of Substitution (CES)]
Valuations of the form
\begin{align*}
    v(y) =&\; \parens{ \sum_{i=1}^n \alpha_i y_i^{\kappa} }^{\beta},
\end{align*}
with each $\alpha_i, \kappa, \beta > 0$ and $\kappa,  \beta \kappa < 1$, are Hölder continuous, differentiable, strictly concave, non-decreasing, and homogeneous over a convex set in $\R^{n}_{+}$. 
\end{example}

\begin{example}[Cobb-Douglas]
Valuations of the form
\begin{align*}
    v(y) =&\;  \prod_{i=1}^n  y_i^{\alpha_i},
\end{align*}
with $\alpha_i > 0$ and $\sum_{i=1}^n \alpha_i < 1$ are Hölder continuous, differentiable, strictly concave, non-decreasing, and homogeneous over a convex set in $\R^{n}_{+}$. 
\end{example}

We initially assume that Assumption \ref{assum:pricing-val} holds over all of $\R_{+}^n$,  but will restrict our attention to the set $\Y \subseteq \R^{n}_{+}$ of bundles where $v(y) \geq \phi \norm{y}$ for each $y \in \Y$, and we note that our results can be extended to arbitrary downward-closed convex sets (where $b y \in \Y$ for any $y \in \Y$ and $b \in (0,1]$).  
In Section \ref{subsec:pricing-analysis} we that show Assumptions \ref{assum:pricing-prod} and \ref{assum:pricing-val} yield several important properties which enable optimization via our framework. We show a unique mapping between price vectors and bundle purchases (for any fixed reserves and consumption rate), that restricting attention to $\Y$ is justified under rationality constraints, and that $\Y$ is convex.

Further, there is some price vector which yields a reserve update to any $y_t \in \Y$ in a neighborhood around $y_{t-1}$, yielding local controllability. 
Crucially, we show that there are concave surrogate rewards $f^*_t(y_t)$ which will closely track our true rewards $f_t(p_t, x_t)$, leveraging the following property of homogeneous functions. 

\begin{proposition}[Euler's Theorem for Homogeneous Functions]\label{prop:euler-homog} A continuous and differentiable function $v : \Y \rightarrow \R_{+}$ is homogeneous of degree $k$ if and only if
\begin{align*}
    \langle \nabla v(y) , y \rangle =&\; k \cdot v(y).
\end{align*}
\end{proposition}

We run $\oenftrl$ directly over these concave surrogate rewards (by inverting the sign of each), where each $p_t$ can be computed efficiently in terms of $y_{t-1}$ and $\theta_t$, and we show that the surrogate reward distance from our true rewards is bounded.
While our rewards will not be Lipschitz over $\Y$ in general, we show that appropriately calibrating our step size yields sublinear regret with dependence on the Hölder continuity parameters. 
We measure our regret with respect to the set of {\it stable reserve policies}, i.e.\ pricing policies where $y_t$ remains constant.

\begin{definition}[Regret for Stable Reserve Policies] Let $\PP_{\Y} = \{P_y : y \in \Y \} $ be the set of stable reserve policies, where for any $y_{t-1}$ and $\theta_t$ satisfying $(1-\theta_t)y_{t-1} \leq y^*$, playing prices computed by a policy $p_t = P_y^*(y_{t-1}, \theta)$ yields
\begin{align*}
    (1-\theta_t)y_{t-1} + x^*(p_t, \theta_t, y_{t-1}) = y^*.
\end{align*}
\end{definition}
It is straightforward to see that any $P_y^* \in \PP_{\Y}$ maintains the invariant that $y_t = y^*$, provided that some such $p_t$ is always feasible.

\subsection{Analysis}\label{subsec:pricing-analysis}
We show a series of results establishing the key conditions allowing us to formulate this problem as a locally controllable instance of online nonlinear control. We first show that any positive bundle is the unique optimal purchase for some positive price vector.
\begin{lemma}\label{lemma:pricing-unique-bundle}
For any reserves $y_{t-1} \in \R_{\geq 0}^n$, consumption rate $\theta_t \in [\theta, 1]$, and vector $y_t \in \R_{+}^n$ where $y_t > (1 - \theta_t)y_{t-1}$ elementwise, the bundle $x_t = y_t - (1 - \theta_t)y_{t-1}$ is the unique solution to 
\begin{align*}
    x_t =&\; x^*(p_t, \theta_{t}, y_{t-1}) 
\end{align*}
for prices $p_t = \nabla v(y_t)$.
\end{lemma}
\begin{proof}
Recall that the consumer's bundle choice is given by
\begin{align*}
    x^*(p_t, \theta_{t}, y_{t-1}) =&\; \argmax_{x \in \R_{+}^n} v(x + (1 - \theta_t)y_{t-1}) - \langle p_t, x \rangle.
\end{align*}
Note that $v((1 - \theta_t)y_{t - 1} + x) - \langle  p_t , x\rangle$ is strictly concave in $x$ for any $x \in \R^{n}_{+}$,  as the gradients 
\begin{align*}
    \nabla_x v((1 - \theta_t)y_{t+1} + x)  =&\; \nabla_{y_t} v(y_t)
\end{align*}
are preserved at each point $y_t = (1 - \theta_t)y_{t+1} + x$, and subtracting the linear function $\langle x, p_t \rangle$ does not affect strict concavity.  We also have that $p_t \in \R_{+}^n$  for  prices $p_t = \nabla v(y_t)$, as $v$ is strictly concave and non-decreasing. This yields that $v((1 - \theta_t)y_{t - 1} + x) - \langle  p_t, x \rangle$ has a unique global maximum at $x_t = y_t - (1 - \theta_t)y_{t-1}$, as $\nabla_x ( v((1 - \theta_t)y_{t+1} + x) - \langle p_t, x \rangle  ) = \mathbf{0}$.
\end{proof}
As such, the $\argmax$ for $x^*(p_t, \theta_{t}, y_{t-1})$ is unique whenever $p_t = \nabla v(y)$ for some $y \in \R_{+}^n$. 
We let $p^*(x_t; y_{t-1}, \theta_t) = \nabla v((1-\theta_t)y_{t-1} + x_t)$ denote this price vector which induces a purchase of $x_t$. 
For any other price vector $p$, the maximizing bundle $x_t$ either approaches a point on the boundary of $\R_{+}^n$, or grows unboundedly. 
We restrict our attention to bundles contained in $\R_{+}^n$, and show that the issue of unboundedness is resolved by rationality considerations for the producer. We characterize the per-round rewards of stable reserve policies as concave functions of $y \in \R_{+}^n$, and show that the optimal such policy corresponds to some state $y^* \in \Y$, where $\Y$ is convex and bounded. 
\begin{lemma}
    The round-$t$ reward of a stable reserve policy $P_{y}$ corresponding to any $y \in \R_{+}^n$ is given by a strictly concave function
    \begin{align*}
        f_t(P_y) =&\; \theta_t  k \cdot v(y)  -  c_t(\theta_t y).
    \end{align*}
\end{lemma}
\begin{proof}
We first note that we can maintain $y_t = y$ in every round by Lemma \ref{lemma:pricing-unique-bundle}, as $y_0 = \mathbf{0}$ and $(1 - \theta_t)y < y$. As such, a bundle $x_t = \theta_t y$ is purchased in each round at prices $\nabla v(y)$, and our reward is given by
\begin{align*}
    f_t(P_y) =&\; f_t(p^*(\theta_t y ; y, \theta_t), \theta_t y) \\
    =&\;   \langle \nabla v(y), \theta_t y \rangle  -  c_t(\theta_t y) \\
    =&\; \theta_t  k \cdot v(y)  -  c_t(\theta_t y), 
\end{align*}
where the final step follows from Proposition \ref{prop:euler-homog} for homogeneous functions. The function $\theta_t  k \cdot v(y)$ is strictly concave, which is preserved upon subtracting the convex function $c_t(\theta_t y)$.
\end{proof}
\begin{lemma}\label{lemma:pricing-y-convex}
The set $\Y = \{y \in \R_{+}^n : v(y) \geq  \phi \norm{y}\}$ is convex.
\end{lemma}
\begin{proof}
Consider any two points $y, y' \in \Y$, and let $y'' = a y + (1 - a)y'$ for any $a \in [0,1]$. Recall that $y^*\in \R_{+}^n$ belongs to $\Y$ if and only if $v(y^*) \geq \phi \norm{y^*}$. By concavity of $v$, we have that
\begin{align*}
    v(y'') =&\; v(a y + (1 - a)y') \\
    \geq&\; a v(y) + (1 - a) v(y') \\
    \geq&\;  \phi \norm{a y} + \phi \norm{(1 - a) y'} \\
    \geq&\; \phi \norm{a y + (1 - a)y'} \\
    =&\; \phi \norm{y''}
\end{align*}
and so  $y'' \in \Y$, yielding convexity of $\Y$.
\end{proof}
\begin{lemma}
For any $z \in \R_{+}^n$ where $z \notin \Y$, there is some $y \in \Y$ such that $f_t(P_y) \geq f_t(P_z)$ for any $\theta_t$ and $c_t$.  
\end{lemma}
\begin{proof}
Consider some $z \notin \Y$ such that $v(z) = \psi \norm{z}$, for $\psi < \phi$, and let $y = \parens{ \frac{\psi}{\phi}}^{1/k} z$. 
By homogeneity of $v$, we have that $v(y) = \frac{\phi}{\psi} v(z)  = \phi \norm{z}$, and so $y \in \Y$ as $\norm{z} > \norm{y}$.
For any round with costs $c_t$ and consumption rate $\theta_t$ we then have that:
\begin{align*}
    f_t(P_y) - f_t(P_z) =&\; \theta_t  k \parens{ v(y)  - v(z) } -  c_t(\theta_t y)  +  c_t(\theta_t z) \\
    =&\; \theta_t  k \parens{  \frac{\psi}{\phi} - 1}\psi \norm{z} -  c_t(\theta_t y)  +  c_t(\theta_t z) \tag{homogeneity of $v$} \\
    \geq&\; \theta_t  k \parens{  \frac{\psi}{\phi} - 1}\psi \norm{z} + \theta_t \phi \norm{z - y} \tag{ lower bound and convexity of $c_t$ } \\ 
    \geq&\; \theta_t  k \parens{  \frac{\psi}{\phi} - 1}\psi \norm{z} + \theta_t \parens{1 - \parens{ \frac{\psi}{\phi}}^{1/k}  } \phi \norm{z} \\  
    \geq&\;  \theta_t \parens{1 -  \frac{\psi}{\phi}  } \phi \norm{z}  - \theta_t  \parens{ 1 -  \frac{\psi}{\phi} }\psi \norm{z}  \tag{$k , \frac{\psi}{\phi} < 1$} \\  
    >&\; 0. \tag{$\phi > \psi$}
\end{align*}
\end{proof}
Thus the optimal $P_y$ for any cost and consumption sequence corresponds to some $y \in \Y$. We can also bound the radius of $\Y$.
\begin{lemma}
Let $V = \max_{y \in \R_{+}^n : \norm{y} = 1} v(y)$. Then, for every $y \in \Y$ we have that
\begin{align*}
    \norm{y} \leq&\; \parens{\frac{V}{\phi}}^{\frac{1}{1 - k}}.
\end{align*}
\end{lemma}
\begin{proof}
Let $y^* = \argmax_{y : \norm{y} = 1} v(y)$, where we have $v(y^*) = V$. Consider the vector $by^*$ for any $b > 0$. By homogeneity of $v$, we have that
\begin{align*}
    v(b y^*) =&\; b^{k} v(y^*) \\
    =&\; b^k V.
\end{align*}
For any $b > \parens{\frac{V}{\phi}}^{\frac{1}{1 - k}}$ we have that
\begin{align*}
    v(b y^*) =&\; \frac{b}{b^{1 - k}} \cdot V \\
    \leq&\; b \phi, 
\end{align*}
where $\norm{b y^*} > b$ and thus $by^* \notin \Y$. This holds for all vectors with norm $b$, as any such vector $z$ will have at most $b^k V$ by homogeneity, which yields the result. 
\end{proof}
The previous result also implies that $by \in \Y$ for any $b < 1$ and $y \in \Y$.  We assume that $V > \phi$, which is without loss of generality as we may otherwise take $\phi$ to be smaller artificially; we assume $\phi$ is small enough to ensure that $\Y$ contains a ball $\B_1(y_1)$ of radius 1 around some $y_1 \in \Y$, and we let $R = \parens{\frac{V}{\phi}}^{\frac{1}{1 - k}}$. We consider the dynamics to be given by
\begin{align*}
    D_t(p_t, y_{t-1}) =&\; (1 - \theta_t)y_{t-1} + x^*(p_t, \theta_t, y_{t-1}). 
\end{align*}
We let $\Z = \R_{+}^n$ denote our action space of price vectors; while dynamics here are not action-linear, we can still compute our desired action $p_t = \nabla v(y_t)$ efficiently, as we assume we have knowledge of $v$. While the dynamics depend on $\theta_t$, our choice of action $p_t$ depends only on the target update $y_t$ to the consumer's reserves, by Lemma \ref{lemma:pricing-unique-bundle}. Further, upon observing $x_t$, we can solve for $\theta_t$ as
\begin{align*}
    \theta_t =&\; 1 -  \frac{y_t - x_t}{y_{t-1}}
\end{align*}
for purposes of representing our surrogate losses, which are given by
\begin{align*}
    f^*_t(y_t) =&\; \theta_t  k \cdot v(y)  -  c_t(\theta_t y).
\end{align*}
We now show that the dynamics satisfy local controllability. 
\begin{lemma}[Local Controllability]
    The instance $(\Z, \Y, D_t)$ satisfies $\theta$-local controllability for each round $t$.
\end{lemma}
\begin{proof}
We show that $\theta$-local controllability holds over all of $\R_{+}^n$, which implies $\theta$-local controllability over $\Y$ as each distance $\pi(y_{t-1})$ while the feasible update region remains the same. By Lemma \ref{lemma:pricing-unique-bundle}, any update where $y_{t} \geq (1 - \theta_t)y_{t-1}$ elementwise is feasible. Each $\pi(y_{t-1})$ over $\R_{+}^n$ is simply the minimum element of $y_t$, which we denote here by $m$. Each element of $y_{t-1}$ is decreased by at least $\theta m$, and so any $y_t$ in the $\ell_{\infty}$ ball of radius $\theta m = \theta \pi(y_{t-1})$, and thus the $\ell_2$ ball of radius $\theta \pi(y_{t-1})$, is feasible.    
\end{proof}
We are now ready to analyse the regret of $\oenftrl$ for the problem. The remaining key issues to resolve will be the errors between our true and surrogate rewards $f_t$ and $f_t^*$, as well as the lack of Lipschitz continuity for our rewards. We will make use of more general formulations of the guarantees of $\ftrl$, (see e.g.\ \cite{hazan2021introduction}).
\begin{proposition}\label{prop:ftrl-holder} For a $\gamma$-strongly convex regularizer $\psi : \Y \rightarrow \R$ where $\abs{\psi(y) - \psi(y')} \leq G$ for all $y, y' \in \Y$, and for convex losses $f_1,\ldots, f_T$, the regret of $\ftrl$ is bounded by
\begin{align*}
    \Reg_{T}(\textup{$\ftrl$}) \leq&\; \sum_{t=1}^T (g_t(y_t) - g_t(y_{t+1}))  + \frac{G}{\eta},
\end{align*}
where $g_t(y) = \langle \nabla_t f_t(y_t) , y \rangle$ and   $g_t(y_t) - g_t(y_{t+1}) \geq \frac{\gamma}{\eta} \norm{y_{t+1} - y_t}^2$.
\end{proposition}
We show that this implies a regret bound for $(\lambda, \beta)$-Hölder continuous convex losses, recovering the $\lambda$-Lipschitz bounds when $\beta = 1$.
\begin{theorem}
For $(\lambda, \beta)$-Hölder continuous convex losses, $\ftrl$ with obtains regret bounded by
\begin{align*}
    \Reg_{T}(\ftrl) \leq&\; T \lambda \parens{ \frac{\eta \lambda }{\gamma } }^{\beta / (2 - \beta)}  + \frac{G}{\eta}
\end{align*}
and chooses points which satisfy $\norm{y_{t+1} - y_t} \leq \parens{ \frac{\eta \lambda}{\gamma} }^{1/(2 - \beta)}$ in each round.
\end{theorem}
\begin{proof}
For $(\lambda, \beta)$-Hölder continuous convex losses $f_t$, we have that 
\begin{align*}
    g_t(y_t) - g_t(y_{t+1}) =&\; \langle \nabla_t f_t(y_t) , y_{t} - y_{t+1} \rangle \\
    =&\; \langle \nabla_t f_t(y_t) , (2 y_{t} - y_{t+1}) - y_t \rangle \\
    \leq&\; f_t(2 y_{t} - y_{t+1}) - f_t(y_t)
\end{align*}
by convexity of $f_t$, where $\norm{(2 y_{t} - y_{t+1}) - y_t} = \norm{y_t - y_{t+1}}$, and so
\begin{align*}
    g_t(y_t) - g_t(y_{t+1}) \leq&\; \lambda \norm{y_t - y_{t+1}}^{\beta}
\end{align*}
by Hölder continuity. Combining with the lower bound on $g_t(y_t) - g_t(y_{t+1})$ from Proposition \ref{prop:ftrl-holder} gives us that 
\begin{align*}
   \frac{\gamma}{\eta} \norm{y_{t+1} - y_t}^2  \leq&\; g_t(y_t) - g_t(y_{t+1}) \leq \lambda \norm{y_t - y_{t+1}}^{\beta}
\end{align*}
and thus
\begin{align*}
    g_t(y_t) - g_t(y_{t+1}) \leq&\; \lambda  \parens{ \frac{\eta \lambda }{\gamma } }^{\beta / (2 - \beta)},
\end{align*}
yielding a regret bound of
\begin{align*}
    \Reg_{T}(\ftrl) \leq&\; T \lambda \parens{ \frac{\eta \lambda }{\gamma } }^{\beta / (2 - \beta)}  + \frac{G}{\eta}
\end{align*}
with per-round distance at most $\norm{y_{t+1} - y_t} \leq \parens{ \frac{\eta \lambda}{\gamma} }^{1/(2 - \beta)}$.
\end{proof}
We note that the concave surrogate rewards $f_t^*(y_t)$ are a sum of a $(k \lambda, \beta)$-Hölder continuous function and a $(L_c, 1)$-Hölder continuous (i.e.\ Lipschitz) function; we assume that each function is $(L, \beta)$-Hölder continuous with $L = k\lambda + L_c$, which is sufficient for for large enough $T$ as we will have $\norm{y_t - y_{t-1}} \leq 1$ and thus $\norm{y_t - y_{t-1}} \leq \norm{y_t - y_{t-1}}^{\beta}$. We use a similar analysis to bound the error between true and surrogate rewards, yielding our regret bound for $\oenftrl$.
\begin{theorem}\label{thm:pricing-app-main}
The regret of $\oenftrl$ with respect to the stable reserve policies $\PP_{\Y}$ is bounded by
\begin{align*}
    \Reg_{T}(\oenftrl) \leq&\; 2L \parens{ \frac{G}{\gamma} }^{\beta / 2} \parens{T \parens{3 + \parens{ \frac{R}{\theta}}^{\beta}}}^{ (2 - \beta)/{2} }.
\end{align*}
\end{theorem}
\begin{proof}
We reparameterize to treat the bundle $y_1$ where $\B_{1}(y_1) \subseteq \Y$ as the origin, and assume the choice of regularizer has $y_1$ as its minimum.  
By Theorem \ref{thm:oenftrl}, for any step size and $\delta > 0$ such that $\norm{y_t - y_{t-1}} \leq \delta\theta$, running $\oenftrl$ for the $\theta$-locally controllable instance $(\Z, \Y, D)$ over the surrogate rewards $f_t^*$, with inradius 1 and radius $R$, obtains 
\begin{align*}
    \sum_{t=1}^T f_t^*(y^*) - \sum_{t=1}^T f_t^*(y_t)  \leq&\; TL(\delta R)^{\beta} +  T L \parens{ \frac{\eta L }{\gamma } }^{\beta / (2 - \beta)}  + \frac{G}{\eta}\\
    \leq&\; T L \parens{1 + \parens{ \frac{R}{\theta}}^{\beta} } \parens{ \frac{\eta L }{\gamma } }^{\beta / (2 - \beta)}  + \frac{G}{\eta}\\
    \leq&\; 2L \parens{ \frac{G}{\gamma} }^{\beta / 2} \parens{T \parens{1 + \parens{ \frac{R}{\theta}}^{\beta}}}^{ (2 - \beta)/{2} } \\
    \overset{\Delta}{=}&\; \Reg_T(f^*)
\end{align*}
for any $y^* \in \Y$, upon setting $\delta = \frac{1}{\theta} \parens{\frac{\eta \lambda }{\gamma }}^{1/(2-\beta)}$ and $\eta = \parens{ \frac{G}{KT} }^{(2 - \beta )/2}$, where $$K^* = L \parens{1 + \parens{ \frac{R}{\theta}}^{\beta} } \parens{ \frac{ L }{\gamma } }^{\beta / (2 - \beta)}.$$

Note that the surrogate rewards exactly track the true rewards when a stable reserve policy $P_{y^*}$ is played, and so our regret with respect to the best stable reserve policy $P_{y^*}$ is at most
\begin{align*}
    \sum_{t=1}^T f_t(P_{y^*}) - \sum_{t=1}^T f_t(y_t)  \leq&\; \Reg_T(f^*) + \sum_{t=1}^T f^*_t(y_t) - f_t(p_t, x_t) \\
    \leq&\; \Reg_T(f^*) + \sum_{t=1}^T \langle \nabla v(y_t) , \theta y_t - x_t \rangle  - c_t(\theta y_t ) + c_t(x_t) \\
    \leq&\; \Reg_T(f^*) + \sum_{t=1}^T (1 - \theta_t) \parens{ \langle \nabla v(y_t) , y_{t-1} - y_t \rangle  + L\norm{y_{t} - y_{t-1}} } \tag{$x_t = (1 - \theta_t)y_{t-1}$}\\
    \leq&\; \Reg_T(f^*) + \sum_{t=1}^T \parens{ \langle \nabla v(y_t) , y_{t} - (2y_t - y_{t-1}) \rangle  + L\norm{y_{t} - y_{t-1}} } \\
    \leq&\; \Reg_T(f^*) + \sum_{t=1}^T  v(y_t) - v(2y_t - y_{t-1})   + L\norm{y_{t} - y_{t-1}}  \tag{concavity of $v$} \\
    \leq&\; \Reg_T(f^*) + \sum_{t=1}^T 2 L\norm{y_{t} - y_{t-1}}^{\beta}  \tag{Hölder, $\norm{y_{t} - y_{t-1}} \leq 1$} \\    
    \leq&\; \Reg_T(f^*) + 2TL \parens{ \frac{ \eta L }{ \gamma} }^{\beta / (2 - \beta)} \\
   \leq&\; 2L \parens{ \frac{G}{\gamma} }^{\beta / 2} \parens{T \parens{3 + \parens{ \frac{R}{\theta}}^{\beta}}}^{ (2 - \beta)/{2} }
\end{align*}
upon updating $K^*$ to $K$ as 
$$K = L \parens{3 + \parens{ \frac{R}{\theta}}^{\beta} } \parens{ \frac{ L }{\gamma } }^{\beta / (2 - \beta)},$$
which yields the theorem.
\end{proof}
Theorem \ref{thm:pricing-body} follows directly from Theorem \ref{thm:pricing-app-main}.

\section{Background and Proofs for Section \ref{subsec:app-steering}: Steering Learners}
\label{sec:appendix-steering}

\subsection{Background}

While much of the literature related to no-regret learning in general-sum games considers either rates of convergence to (coarse) correlated equilibria \cite{blum2008regret,anagnostides2022near} or welfare guarantees for such equilibria \cite{robustPOA,hartline2015no}, a recent line of work \cite{DBLP:journals/corr/abs-1711-09176,deng2019strategizing,mansour2022strategizing} has considered the question of {\it optimizing} one's reward when playing against a no-regret learner. A target benchmark which has emerged for this problem is the value of the {\it Stackelberg} equilibrium of a game (the optimal mixed strategy to ``commit to'', assuming an opponent best responds), which was shown by attainable by \cite{deng2019strategizing} against any no-regret algorithm and optimal in many cases (e.g.\ for no-swap learners), both up to $o(T)$ terms, and further which may yield higher reward for the optimizer than (coarse) correlated equilibria.

We show a class of instances for which the problem for optimizing reward against a learner playing according to gradient descent can be formulated as a locally controllable instance of online nonlinear control with adversarial perturbations and surrogate state-based losses. The simplest non-trivial instances we consider are those where the optimizer's reward is a function only of the learner's actions (i.e.\ all rows of their reward matrix are identical), and the optimization problem amounts to {\it steering} the learner to a desired strategy via one's choice of actions. Additionally, we allow the game matrices to change over time, which has not been substantially considered in prior work to our knowledge. We require that the learner's matrices do not change too quickly (which we model as adversarial disturbances to dynamics), and the optimizer's matrices can change arbitrarily provided that they remain close to {\it some} row-identical matrix (which we model as imprecision in our surrogate loss function).

\subsection{Model}
Here we are tasked with playing a sequence of bimatrix games against a no-regret learning opponent, where the game matrices may change adversarially in each round.
We assume the following properties hold for the adversarial sequence of games.
\begin{assumption}\label{assum:game-props}
   For a sequence $\{(A_t, B_t) : t \in [T] \}$ of $m\times n$ bimatrix games, with $m > n$:
\begin{itemize}
    \item Each entry of $A_t$ and $B_t$ lies in $[-\frac{L}{2\sqrt{n}}, \frac{L}{2\sqrt{n}} ]$
    \item the convex hull of the of the rows of each $B_t$ contains the unit ball in $\R^n$, 
    \item $\norm{x A_{t} - xA^*_t} \leq \delta_t$ for any $x \in \Delta(m)$, where each row of $A^*_t$ is identical, and
    \item $\norm{x B_{t} - x B_{t-1}} \leq \epsilon_t$ for any $x \in \Delta(m)$.
\end{itemize} 
\end{assumption}

Each game $(A_t, B_t)$ is revealed after Players A and B commit to their respective strategies $x_t \in \Delta(m)$ and $y_t \in \Delta(n)$. Observe that due to the first property, for any $z \in \B_{1}(\mathbf{0})$, there is some $x \in \Delta(m)$ such that $xB = z$. By the second property, we have that $xA^*_t = x'A^*_t$ for any $x, x' \in \Delta(m)$.

We recall the Online Gradient Descent algorithm with convex losses $\ell_t$ from \cite{opgd}.
\begin{algorithm}
\caption{Online Gradient Descent (OGD)}\label{alg:opgd}
\begin{algorithmic}
\STATE Input: Convex set $\Y \subseteq \R^n$, initial point $y_1 \in \Y$, and step sizes $\theta_1, \ldots, \theta_T$. 
\FOR{$t = 1$ to $T$}
\STATE Play $y_t$ and observe loss $\ell_t(y_t)$.
\STATE Set $\nabla_t = \nabla \ell_t(y_t)$.
\STATE Set $y_{t+1} = \Pi_{\Y} \parens{ y_t - \theta_t \nabla_t} = \text{argmin}_{y \in \Y}\norm{ y_t - \theta_t \nabla_t - y}$. 
\ENDFOR
\end{algorithmic}
\end{algorithm}

\begin{proposition}[\cite{opgd}]
For differentiable convex losses $\ell_t : \Y \rightarrow \R$, with $\theta_{t+1} \leq \theta_t$ for each $t \leq T$, then for all $y^* \in \Y$ the regret of $\textup{OGD}$ is bounded by
\begin{align*}
    \sum_{t=1}^T \ell_t(y_t) - \ell_t(y^*) \leq&\; \frac{2R^2_B}{\theta_T} + \sum_{t=1}^T \frac{\theta_t}{2} \norm{\nabla_t}^2,
\end{align*}
where $R_B$ is the radius of $\Y$. If $\norm{\nabla_t} \leq G_B$ and $\theta_t = \frac{2R_B}{G_B\sqrt{T}}$ for all $t \leq T$, we have that
\begin{align*}
    \sum_{t=1}^T \ell_t(y_t) - \ell_t(y^*) \leq&\; 2R_B G_B \sqrt{T}.
\end{align*}
\end{proposition}

We assume that Player B plays according to OPGD in our setup, with $y_1 = \mathbf{u}_n$ and $\theta = \frac{R_B}{G_B\sqrt{T}}$.
At each round $t$, we (Player A) choose some mixed strategy $x_t \in \Delta(n)$, and Player B plays some mixed strategy $y_t \in \Delta(n)$. Utilities for each player are given by the game $(A_t, B_t)$ as
\begin{align*}
    u_t^A(x_t, y_t) =&\; x_t A_t y_t; \\
    u_t^B(x_t, y_t) =&\; x_t B_t y_t.
\end{align*}
Note that the loss gradient $-\nabla u_t^B (x_t, y_t)$ each round for Player B (for negative utilities) is given by
\begin{align*}
    \nabla_t =&\; - x_t B,
\end{align*}
and so their mixed strategy is updated at each round according to
\begin{align*}
    y_t =&\; \Pi_{\Delta(n)} \parens{ y_{t-1}  + \theta (x_{t-1} B_{t-1} ) }.
\end{align*}

Our utility is given by $x_t A_t y_t = \mathbf{u}_n A_t^*  y_t + x_t(A_t - A_t^*) y_t $, as $x_t$ does not affect rewards from $A_t^*$.
We benchmark the regret of an algorithm $\A$ against the optimal profile $(x, y) \in \Delta(m) \times \Delta(n)$: 
\begin{align*}
    \Reg_{T}(\A) =&\; \max_{(x, y) \in \Delta(m) \times \Delta(n)}  \sum_{t=1}^T x A_t y -  x_t A_t y_t.
\end{align*}

Note that the per-round average utility for the maximizing $(x, y)$ is at least as high as that obtained by the Stackelberg equilibrium of the average game $\parens{\sum_t \frac{A_t}{T}, \sum_{t} \frac{B_t}{T}}$, as for this objective one can choose both players' strategies without restriction. We remark that finding the Stackelberg equilibrium for any fixed game $(A_t^*, B_t)$ in our setting, where $A_t^*$ has identical rows, is straightforward: it suffices to optimize over $[n]$, as any fixed action $j \in [n]$ is a best response to some $x \in \Delta(m)$ by our assumption on the rows of $B_t$, and as our rewards are only a function of Player B's strategy $y$. However, we are not aware of any prior work which enables competing with the average-game Stackelberg value against a learning opponent when games arrive online.

\subsection{Analysis}

We first show that the problem can be formulated via known, strongly $\theta$-locally controllable dynamics with adversarial disturbances. As $B_t$ changes slowly between rounds, we can run $\oenftrluap$ with disturbances representing the error resulting from assuming that  $B_t$ does not change from $B_{t-1}$.

\begin{lemma}\label{lemma:games-strictly}
Given the knowledge available prior to selecting $x_t$, updates for $y_{t+1}$ can be expressed via known action-linear dynamics $(\X, \Y, D_t)$ which satisfy strong $\theta$-local controllability, and with adversarial disturbances $w_t$ satisfying $\sum_{t=1}^T \norm{w_t} \leq  \theta \sum_{t=1}^T \epsilon_t$. 
\end{lemma}
\begin{proof}
First, note that we can compute Player B's current strategy $y_t$, as it is a function only of games and strategies up to round $t-1$, all of which are observable. Given the update rule for $\opgd$, we can formulate the dynamics ${D}_t(x_t, y_t)$ update as
\begin{align*}
{D}_t(x_t, y_t) =&\; \Pi_{\Delta(n)} \parens{ y_t  + \theta(x_t B_{t} ) } \\
=&\; \Pi_{\Delta(n)} \parens{ y_t  + \theta(x_t B_{t-1} )  + \theta(x_t(B_t - B_{t-1}))} \\
=&\; \Pi_{\Delta(n)} \parens{ y_t  + \theta(x_t B_{t-1} ) } + w_{t}
\end{align*}
where $w_t$ represents the error from assuming $B_t = B_{t-1}$. by standard properties of Euclidean projection, and the change bound on $B_t$,  we have that $\norm{w_t} \leq \norm{\theta(x_t(B_t - B_{t-1}))} \leq \theta \epsilon_t$ . Further, the update is action-linear (up to projection, prior to $w_t$).

To see that $D_{t}$ satisfies strong $\theta$-local controllability, we recall that the convex hull of the rows of $B_{t-1}$ contain the unit ball, and  so for any $y^*$ in $\B_{\theta}(y_t) \cap \Delta(n)$ there is some $x_t \in \Delta(m)$ such that $\theta(x_tB_{t-1}) = y^* - y_t$.
\end{proof}

At round each round $t$, our loss is given by $f_t(x_t, y_t) = -x_t A_t{y}_t$.
There are two barriers to running our algorithm. First, the update for $y_t$ is determined by $x_{t-1}$ and not $x_t$, yet we do not see $A_{t-1}$ prior to selecting $x_{t-1}$, which would be required to take the appropriate step following $f_{t-1}$. Second, the loss depends on $x_t$ in addition to $y_t$.
To address both issues,  
we instead run $\oenftrluap$ with surrogate losses $\tilde{f}_{t}(\tilde{y}_{t}) = -\mathbf{u}_n A_{t-1} {y}_{t}$, with action rounds relabeled to account for the fact that $x_{t-1}$ influences the step for $y_t$ (which does not change the behavior of the algorithm). We set $A_0 = \mathbf{0}_{m,n}$.

\begin{theorem}\label{thm:games-appendix}
Repeated play against an opponent using $\opgd$ with step size $\theta = \Theta(T^{-1/2})$ in a sequence of games $(A_t, B_t)$ satisfying Assumption \ref{assum:game-props} can be cast as an instance of online control with strongly $\theta$-locally controllable dynamics, for which the regret of $\oenftrluap$ is at most
\begin{align*}
    \Reg_T(\oenftrluap) \leq&\; O\parens{\sqrt{T} + \sum_{t=1}^T (\delta_t + \epsilon_t)},
\end{align*}
with efficient per-round computation.
\end{theorem}
\begin{proof}
We first analyze regret with respect to the surrogate losses $\tilde{f}_t(y_t)$.
To run $\oenftrluap$ for $\alpha > 0$, it suffices to calibrate the step size for the internal $\ftrl$ instance such that $\eta \frac{L}{\gamma} \leq \theta \alpha$. 
Given that rewards are bounded in $[-\frac{L}{2\sqrt{n}}, \frac{L}{2\sqrt{n}}]$, we have that each $x_tB_ty_t$ is $\frac{L}{\sqrt{n}}$-Lipschitz for the $\ell_1$ norm, and thus $L$-Lipschitz for the $\ell_2$ norm, so we can take $G_B = L$. Further, the $\ell_2$ radius of $\Delta(n)$ is $R_B = {\sqrt{2}} / {2}$, and so we have that $$\theta  =  \sqrt{\frac{2}{L^2T}}.$$
Then, for a strongly $\theta$-locally controllable instance with total perturbation bound $\sum_{t=1}^T \norm{w_t} \leq E$, we obtain the regret bound
\begin{align*}
    \Reg_T(\oenftrluap) \leq&\; \eta \frac{TL^2}{\gamma} + \frac{G}{\eta} + \frac{2LRE}{(1 - \alpha)\theta} \tag{Thm.\ \ref{thm:oenftrluap-regret}}
\end{align*}
for any \begin{align*}
    \eta \leq \min\parens{ \sqrt{ \frac{G\gamma}{L^2 T} },  \alpha \sqrt{ \frac{2}{T} }}.
\end{align*}
By Lemma \ref{lemma:games-strictly}, we can efficiently run $\oenftrluap$ over the surrogate losses $\tilde{f}_t$ and bound regret with respect to any $y^* \in \Y$ as:
\begin{align*}
    \sum_{t=1}^T \tilde{f}_t(y_t) - \tilde{f}_t(y^*) \leq&\; \eta \frac{TL^2}{\gamma} + \frac{G}{\eta} +  \frac{\sqrt{2}L \cdot \sum_{t=1}^T \epsilon_t}{1 - \alpha}.
\end{align*}
Further, we can bound the error from the surrogate losses as 
\begin{align*}
    \sum_{t=1}^T {f}_t(x_t, y_t) - \tilde{f}_t(y_t) =&\; \sum_{t=1}^T  {f}_t(x_t, y_t) - {f}_{t-1}(\mathbf{u}_n, y_t) \\
    \leq&\; \frac{L}{2\sqrt{n}} + \sum_{t=1}^{T-1}  {f}_t(x_t, y_t) - {f}_{t}(\mathbf{u}_n, y_{t+1}) \tag{$f_0(\mathbf{u}_n, y_1) = 0$, $f_T(x_T, y_T) \leq \frac{L}{2\sqrt{n}}$} \\
    \leq&\; \frac{L}{2\sqrt{n}} + \eta \frac{TL^2}{\gamma} +  \sum_{t=1}^{T-1}  x_t(A_t - A_t^*) y_t \tag{Prop. \ref{prop:ftrl-step}} \\
    \leq&\; \frac{L}{2\sqrt{n}} + \eta \frac{TL^2}{\gamma} +  \sum_{t=1}^{T} \delta_t, \tag{Assumption \ref{assum:game-props}, Cauchy-Schwarz}
\end{align*}
and likewise, for any $(x^*, y^*) \in \Delta(m) \times \Delta(n)$ we can bound
\begin{align*}
\sum_{t=1}^T \tilde{f}_t(y^*) - f_t(x^*, y^*) \leq&\; -f_T(x^*, y^*) - \sum_{t=1}^{T-1} x^* (A_t - A_t^*) y^* \\
\leq&\; \frac{L}{2\sqrt{n}} +  \sum_{t=1}^{T} \delta_t.
\end{align*}

Combining the previous results, we have that for any  $(x^*, y^*) \in \Delta(m) \times \Delta(n)$, the regret of $\oenftrluap$ with respect to the true losses is bounded by
\begin{align*}
    \sum_{t=1}^T f_t(x_t, y_t) - f_t(x^*, y^*) \leq&\; \sum_{t=1}^T \tilde{f}_t(\tilde{y}_{t}) - \tilde{f}_t(y^*) + \sum_{t=1}^T {f}_t(x_t, y_t) - \tilde{f}_t(y_t) + \sum_{t=1}^T \tilde{f}_t(y^*) - f_t(x^*, y^*)\\
    \leq&\; \eta \frac{2 TL^2}{\gamma} + \frac{G}{\eta} +  \frac{L}{\sqrt{n}} +  2\sum_{t=1}^{T} \delta_t + \frac{\sqrt{2}L \cdot \sum_{t=1}^T \epsilon_t}{1 - \alpha}  \\ 
    \leq&\; 3 \cdot \max\parens{\sqrt{\frac{TGL^2}{\gamma}}, \sqrt{\frac{T}{2\alpha^2}}} +  \frac{L}{\sqrt{n}} +  2\sum_{t=1}^{T} \delta_t + \frac{\sqrt{2}L \cdot \sum_{t=1}^T \epsilon_t}{1 - \alpha} 
\end{align*}
for any $\alpha \in (0,1)$, which yields the theorem.
\end{proof}
Theorem \ref{thm:games-body} follows directly from Theorem \ref{thm:games-appendix}.

\end{document}